\documentclass[twoside,11pt]{article}

\usepackage[preprint]{jmlr2e}

\jmlrheading{}{}{}{6/20}{}{kalweit}{Gabriel Kalweit, Maria Huegle and Joschka Boedecker}

\ShortHeadings{Composite Q-learning}{Kalweit, Huegle and Boedecker}
\firstpageno{1}

\usepackage{hyperref}
\usepackage{url}

\usepackage{amsmath, amssymb}
\usepackage{tikz}
\usetikzlibrary{automata, arrows, arrows.meta,shapes,shadows,fit,positioning,calc,matrix,decorations.pathreplacing, patterns}
\usepackage{booktabs}
\usepackage{subcaption}
\usepackage{graphicx}
\usepackage{soul}
\usepackage[ruled, vlined, linesnumbered]{algorithm2e}
\let\oldnl\nl
\newcommand{\nonl}{\renewcommand{\nl}{\let\nl\oldnl}}
\usepackage{setspace}
\usepackage{subcaption}
\usepackage{cleveref}
\usepackage{xcolor}
\usepackage{stmaryrd}
\usepackage{multicol}
\usepackage{wrapfig}
\usepackage{bm}

\DeclareMathOperator*{\argmax}{arg\,max}

\begin{document}
\title{Composite Q-learning: Multi-scale Q-function Decomposition and Separable Optimization}

\author{\name Gabriel Kalweit \email kalweitg@cs.uni-freiburg.de \\
       \addr Neurorobotics Lab\\
       Department of Computer Science\\
       University of Freiburg,
       Germany
       \AND
       \name Maria Huegle \email hueglem@cs.uni-freiburg.de \\
       \addr Neurorobotics Lab\\
       Department of Computer Science\\
       University of Freiburg, Germany
       \AND
       \name Joschka Boedecker \email jboedeck@cs.uni-freiburg.de \\
       \addr Neurorobotics Lab\\
       Department of Computer Science\\
       University of Freiburg, Germany}

\editor{}

\maketitle

\begin{abstract}
In the past few years, off-policy reinforcement learning methods have shown promising results in their application for robot control. Deep Q-learning, however, still suffers from poor data-efficiency and is susceptible to stochasticity in the environment or reward functions which is limiting with regard to real-world applications. We alleviate these problems by proposing two novel off-policy Temporal-Difference formulations: (1) Truncated Q-functions which represent the return for the first $n$ steps of a target-policy rollout w.r.t. the full action-value and (2) Shifted Q-functions, acting as the farsighted return after this truncated rollout. This decomposition allows us to optimize both parts with their individual learning rates, achieving significant learning speedup. We prove that the combination of these short- and long-term predictions is a representation of the full return, leading to the Composite \mbox{Q-learning} algorithm. We show the efficacy of Composite Q-learning in the tabular case and compare Deep Composite Q-learning with TD3 and TD3($\Delta$), which we introduce as an off-policy variant of TD($\Delta$). Moreover, we show that Composite TD3 outperforms TD3 as well as state-of-the-art compositional Q-learning approaches significantly in terms of data-efficiency in multiple simulated robot tasks and that Composite Q-learning is robust to stochastic environments and reward functions.
\end{abstract}

\begin{keywords}
  Compositional Value-learning, Noise Robustness, Data-efficient Reinforcement Learning, Off-policy Learning, Model-free Reinforcement Learning
\end{keywords}

\section{Introduction}
In recent years, Q-learning \citep{watkins1992q} has achieved major successes in a broad range of areas by employing deep neural networks \citep{mnih2015human,silver2018general,ddpg}, including environments of higher complexity \citep{learningbyplaying} and even in first real world applications \citep{learningtowalk}. Due to its off-policy update, Q-learning can leverage transitions collected by any policy which makes it more data-efficient compared to on-policy methods. Deep Q-learning, however, still has a very high demand for data samples and suffers from stability issues which is limiting with regard to robot applications.\\

One reason for these issues is the long temporal horizon the reward signal has to propagate through, possibly depending on many decision steps and transitions. In addition, the immediate reward signal can be very noisy for real-world tasks. To avoid these problems, the full return can be composed of value functions with increasing discount, an approach called TD($\Delta$) \citep{timescales}. Increasing discount factors translate to growing task horizons, leading to a decomposition of the long-term prediction task into multiple subtasks of smaller complexity. The first of these is easiest to predict since it covers the shortest horizon. All other predictions then build upon the first to estimate the remainder of the long-term value. The advantage then results from the fact that values for the easier predictions can potentially stabilize more quickly, leading to faster and more stable learning in the bootstrapping scheme. \citeauthor{timescales} formalize a Bellman-operator over the differences between value functions of increasing discount values in an on-policy setting. 

In this work, we propose an approach that follows a similar reasoning, but breaks down the long-term return into a composition of several short-term predictions over a \emph{fixed temporal horizon} rather than on a discount-level. We approach this via a new formulation of consecutive bootstrapping from value functions corresponding to different horizons of the target-policy associated to the \emph{full return}.\\

This way of composing the long-term return bears a strong resemblance to a combination of Monte Carlo rollouts and a bootstrap of the value function, as common in on-policy methods such as PPO \citep{schulman2017proximal}. To employ $n$-step Monte Carlo rollouts in \textit{off-policy} settings, however, these subtrajectories of the exploratory policy will differ from the true value of the target-policy. In order to take advantage of the temporal relaxation, the replay buffer has to be restricted in size or the rollout length has to be set to a small value to keep the samples close to the target-policy~\citep{d4pg,rainbow}. To avoid these problems, a dynamics model can be used for imaginary rollouts, the so-called \textit{Model-based Value Expansion} (MVE) \citep{mbve, buckman2018sample}. Our approach does not rely on a model of the environment and is off-policy by design. In the \textit{Hybrid Reward Architecture} (HRA), \citet{DBLP:conf/nips/SeijenFLRBT17} suggest a decomposition of the reward and the estimation of value functions for each part of this decomposition which are then combined as an approximation of the full return. HRA addresses the problem of complex rewards and is thus complementary to our work focusing on long time scales. A first formulation of the decomposition of the long-term return into short-term predictions of target-policy rollouts has been done in \citep{kalweit2019composite}. Concurrently, \citet{asis2019} introduced \textit{Fixed Horizon TD-learning}, formalizing action-value functions for different horizons in time. In contrast to our work, however, \citeauthor{asis2019} define a consecutive bootstrapping formulation w.r.t. the target policies of the different truncated horizons and not the full return which can lead to suboptimal behaviour on the long-term. A first general description of using intermediate predictions in Temporal Differences to restrict predictions to a fixed temporal horizon has been introduced already in Sutton's seminal work in \citep{sutton1988learning}. An interesting perspective of our approach to estimate long-term values at different time-scales is to see them as natural auxiliary tasks \citep{DBLP:conf/iclr/JaderbergMCSLSK17}. Also related are \textit{General Value Function Networks} which constrain the hidden state of a recurrent neural network to be multi-step predictions to enhance temporal consistency \citep{schlegel2018general}. Extensions of our work presented here used the formulation of truncated value-functions for the prediction of auxiliary costs within a constrained model-free RL framework \citep{kalweit2020interpretable} and the formulation of Shifted Q-functions as a means to make a novel inverse reinforcement learning method model-free \citep{kalweit2020inverse}.\\

Our contributions are threefold. First, we introduce the \textit{Composite Q-learning} algorithm. For its formulation, we define \textit{Truncated Q-functions}, representing the return for the first $n$ steps of a target-policy rollout w.r.t. the full action-value. In addition, we introduce \textit{Shifted Q-functions} which represent the farsighted return after this truncated rollout. Both are then combined in a mutual recursive definition of the Q-function for the final algorithm. Second, we extend Composite Q-learning to the Deep Learning case and introduce a new entropy-based regularization. Third, we introduce \textit{TD3$(\Delta)$}, an off-policy extension of TD($\Delta$) to deep Q-learning.\\

We describe the theoretical background in \Cref{sec:background} and define Composite Q-learning and TD3($\Delta$) in \Cref{sec:lstq}. We analyze Composite Q-learning in the tabular case and in the function approximation setting for three continuous robot tasks in \Cref{sec:experiments}. We then conclude in \Cref{sec:conclusion}.

\section{Background}
\label{sec:background}
We consider tasks modelled as Markov decision processes (MDP), where an agent executes action $a_t\in\mathcal{A}$ in some state $s_t\in\mathcal{S}$ following its stochastic policy $\pi:\mathcal{S}\times\mathcal{A}\mapsto[0, 1]$. According to the dynamics model $\mathcal{M}:\mathcal{S}\times\mathcal{A}\times\mathcal{S}\mapsto[0, 1]$ of the environment, the agent transitions into some state $s_{t+1}\in\mathcal{S}$ and receives scalar reward $r_{t}$. The agent aims at maximizing the expected long-term return: \begin{equation}\mathcal{R}^\pi(s_t)=\mathbf{E}_{a_{j\geq t}\sim\pi,s_{j>t}\sim\mathcal{M}}\left[\sum_{j=t}^{T-1}\gamma^{j-t}r_j\middle|s_t\right],\end{equation} where $T$ is the (possibly infinite) temporal horizon of the MDP and $\gamma\in[0,1]$ the discount factor. It therefore tries to find $\pi^*$, s.t. $\mathcal{R}^{\pi^*}\geq\mathcal{R}^\pi$ for all $\pi$. If the model of the environment is unknown, model-free methods based on the Bellman Optimality Equation over the so-called action-value: \begin{equation}Q^\pi(s_t, a_t)=\mathbf{E}_{a_{j>t}\sim\pi,s_{j>t}\sim\mathcal{M}}\left[\sum_{j=t}^{T-1}\gamma^{j-t}r_j\middle|s_t, a_t\right],\end{equation}
can be used:
\begin{equation}Q^*(s_t, a_t)=r_t+\gamma\max_a\mathbf{E}_{s_{j>t}\sim\mathcal{M}}\left[Q^*(s_{t+1}, a)\right]. \end{equation}
In the following, we abbreviate $\mathbf{E}_{a_{j>t}\sim\pi,s_{j>t}\sim\mathcal{M}}[\cdot|s_t,a_t]$ by $\mathbf{E}_{\pi,\mathcal{M}}[\cdot]$. One of the most popular model-free reinforcement learning methods is the sampling-based off-policy \textit{Q-learning} algorithm \citep{watkins1992q} for discrete action spaces:
\begin{equation}
    Q(s_t,a_t)\leftarrow (1-\alpha) Q(s_t,a_t) + \alpha(r_t+\gamma \max_aQ(s_{t+1}, a)),
\end{equation}
with learning rate $\alpha$. A representative of continuous model-free reinforcement learning with function approximation is the \textit{Deep Deterministic Policy Gradient} (DDPG) actor-critic method \citep{ddpg}. In DDPG, actor $\mu$ is a deterministic mapping from states to actions, $\mu:\mathcal{S}\mapsto\mathcal{A}$, representing the actions that maximize the critic $Q^\mu$, i.e. $\mu(s_t)=\argmax_a Q^\mu(s_t, a)$. $Q$ and $\mu$ are estimated by function approximators $Q(\cdot, \cdot|\theta^Q)$ and $\mu(\cdot|\theta^\mu)$, parameterized by $\theta^Q$ and $\theta^\mu$. The critic is optimized on the mean squared error between predictions $Q(s_j, a_j|\theta^Q)$ and targets $y_j=r_j+\gamma Q'(s_{j+1}, \mu^\prime(s_{j+1}|\theta^{\mu\prime})|\theta^{Q\prime})$, where $Q^\prime$ and $\mu^\prime$ are target networks, parameterized by $\theta^{Q\prime}$ and $\theta^{\mu\prime}$. The parameters of $\mu$ are optimized following the deterministic policy gradient theorem~\citep{silver2014deterministic}: \begin{equation}\nabla_{\theta^\mu}\mapsfrom\frac{1}{m}\sum_j\nabla_a Q(s,a|\theta^Q)|_{s=s_j, a=\mu(s_j|\theta^\mu)}\nabla_{\theta^\mu}\mu(s|\theta^\mu),\end{equation} and the parameters of the target networks are updated via Polyak averaging: \begin{equation}\theta^{Q\prime}\mapsfrom(1-\tau)\theta^{Q\prime}+\tau\theta^{Q} \text{ and } \theta^{\mu\prime}\mapsfrom(1-\tau)\theta^{\mu\prime}+\tau\theta^{\mu},\end{equation} with $\tau\in[0,1]$.\\

The state-of-the-art actor-critic method TD3 \citep{fujimoto2018addressing} then adds three adjustments to vanilla DDPG. First, the minimum prediction of two distinct critics is taken for target calculation to alleviate overestimation bias, an approach belonging to the family of Double Q-learning algorithms \citep{doubleq}. Second, Gaussian smoothing is applied to the target-policy, addressing the variance in updates. Third, actor and target networks are updated every $d$-th gradient step of the critic, to account for the problem of moving targets.

\section{Combining Long-term and Short-term Predictions in Q-learning}
\label{sec:lstq}
In the following, we introduce the Composite Q-learning algorithm, along with TD3($\Delta$) as an additional baseline. Both approaches aim at decomposing the long-term value into values of smaller time scales. Composite Q-learning approaches this by dividing the bootstrapping into a short-term and a long-term component. TD3($\Delta$) on the other hand formalizes a delta function to estimate the remainder of an action-value corresponding to a smaller discount factor.

\subsection{Composite Q-learning}
\label{subsec:composite}
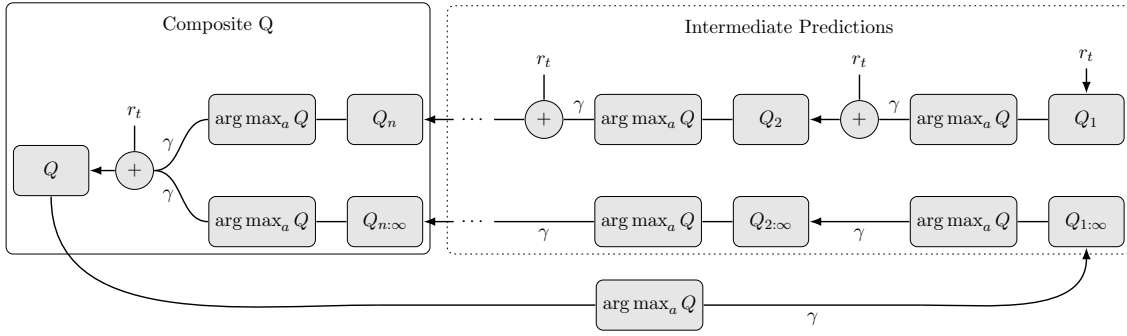
\begin{figure}
  \centering
  \resizebox{\textwidth}{!}{

\begin{tikzpicture}
\tikzset{>=Latex, line width=1pt}

\tikzstyle{box_solid}=[rectangle, draw, rounded corners]
\tikzstyle{box_dashed}=[rectangle, draw, rounded corners, dashed]
\tikzstyle{box_dotted}=[rectangle, draw, rounded corners, dotted]
\tikzstyle{element}=[draw, fill=black!10, text centered] 
\tikzstyle{element_rect}=[element, rounded corners, rectangle, minimum height = 1cm, minimum width = 1.5cm]
\tikzstyle{element_ell}=[element, ellipse]
\tikzstyle{element_circ}=[element, circle]
\tikzstyle{element_dia}=[element, diamond]
\tikzstyle{element_rect_split}=[element, rectangle split, rectangle split horizontal, rectangle split parts=2]

\node (tr1) [element_rect] {$Q_1$};
\node (tr11) [element_rect, left=0.6cm of tr1] {$\argmax_{a}Q$};
\node (tr2) [element_circ, left=0.6cm of tr11] {$+$};
\node (tr3) [element_rect, left=0.6cm of tr2] {$Q_2$};
\node (tr33) [element_rect, left=0.6cm of tr3] {$\argmax_{a}Q$};
\node (tr333) [element_circ, left=0.6cm of tr33] {$+$};
\node (tr4) [left=0.6cm of tr333, minimum height=1cm] {$\cdots$};
\node (tr5) [element_rect, left=0.6cm of tr4] {$Q_{n}$};
\node (tr55) [element_rect, left=0.6cm of tr5] {$\argmax_{a}Q$};

\node (r1) [above=0.5cm of tr1] {$r_t$};
\node (r2) [above=0.5cm of tr2] {$r_t$};
\node (r3) [above=0.5cm of tr333] {$r_t$};

\node (lt1) [element_rect, below=1cm of tr1] {$Q_{1:\infty}$};
\node (lt11) [element_rect, below=1cm of tr11] {$\argmax_{a}Q$};
\node (lt3) [element_rect, below=1cm of tr3] {$Q_{2:\infty}$};
\node (lt33) [element_rect, below=1cm of tr33] {$\argmax_{a}Q$};
\node (lt5) [element_rect, below=1cm of tr5] {$Q_{n:\infty}$};
\node (lt4) [right=0.6cm of lt5] {$\cdots$};
\node (lt55) [element_rect, below=1cm of tr55] {$\argmax_{a}Q$};

\node (mid) at ($(tr55)!0.5!(lt55)$) {};
\node (qplus) [element_circ, left=2cm of mid] {$+$};
\node (q) [element_rect, left=0.5cm of qplus] {$Q$};
\node (rq) [above=0.5cm of qplus] {$r_t$};

\draw [] (tr1.west) edge[-, out=180, in=0, thick] node [above] {} (tr11.east);
\draw [] (tr11.west) edge[-, out=180, in=0, thick] node [above] {$\gamma$} (tr2.east);
\draw [] (tr2.west) edge[->, out=180, in=0, thick] node [above] {} (tr3.east);
\draw [] (tr3.west) edge[-, out=180, in=0, thick] node [above] {} (tr33.east);
\draw [] (tr33.west) edge[-, out=180, in=0, thick] node [above] {$\gamma$} (tr333.east);
\draw [] (tr333.west) edge[-, out=180, in=0, thick] node [above] {} (tr4.east);
\draw [] (tr4.west) edge[->, out=180, in=0, thick] node [above] {} (tr5.east);

\draw [] (r1.south) edge[->, thick] (tr1.north);
\draw [] (r2.south) edge[-, thick] (tr2.north);
\draw [] (r3.south) edge[-, thick] (tr333.north);
\draw [] (rq.south) edge[-, thick] (qplus.north);

\draw [] (lt1.west) edge[-, out=180, in=0, thick] node [below] {} (lt11.east);
\draw [] (lt11.west) edge[->, out=180, in=0, thick] node [below] {$\gamma$} (lt3.east);
\draw [] (lt3.west) edge[-, out=180, in=0, thick] node [below] {} (lt33.east);
\draw [] (lt33.west) edge[-, out=180, in=0, thick] node [below] {$\gamma$} (lt4.east);
\draw [] (lt4.west) edge[->, out=180, in=0, thick] node [below] {} (lt5.east);

\draw [] (tr5.west) edge[-, out=180, in=0, thick] node [left] {} (tr55.east);
\draw [] (tr55.west) edge[-, out=180, in=0, thick] node [left] {$\gamma$} (qplus.east);
\draw [] (lt5.west) edge[-, out=180, in=0, thick] node [left] {} (lt55.east);
\draw [] (lt55.west) edge[-, out=180, in=0, thick] node [left] {$\gamma$} (qplus.east);
\draw [] (qplus.west) edge[->, out=180, in=0, thick] (q.east);

\node (mid0) [minimum height=1cm] at ($(lt1)!0.5!(lt3)$) {};
\node (p1) [below=1cm of lt5] {};
\node (p2) [below=1cm of mid0] {};
\node (mid00) [minimum height=1cm, element_rect] at ($(p1.west)!0.5!(p2.west)$) {$\argmax_{a}Q$};
\draw [] (q.south) edge[-, out=270, in=180, thick] (p1.west);
\draw [] (p1.west) edge[-, out=0, in=180, thick] node [below] {} (mid00.west);
\draw [] (mid00.east) edge[-, out=0, in=180, thick] node [below] {$\gamma$} (p2.west);
\draw [] (p2.west) edge[->, out=0, in=270, thick] (lt1.south);


\node (mid1) at ($(tr5)!0.5!(q)$) {};
\node (mid2) [] at (mid1.north |- r2) {};

\node (mid6) [minimum height=1cm] at (tr1.east) {};
\node (mid7) [minimum height=1cm] at (tr4.west) {};
\node (mid3) [above=0.5cm of mid6] {};
\node (mid4) [above=0.5cm of mid7] {};
\node (mid5) at ($(mid3)!0.5!(mid4)$) {};

\node (Aux-Label) [above=0.27cm of mid5] {Intermediate Predictions};
\node (Aux) [box_dotted, fit=(r2) (tr4) (lt1) (Aux-Label)] {};
\node (Qcom-Label) [above=0.27cm of mid2] {Composite Q};
\node (Qcom) [box_solid, fit=(tr5) (lt5) (q) (rq) (Qcom-Label)] {};
\end{tikzpicture}
  }
  
  \caption{Structure of Composite Q-learning. $Q_i$ denote the Truncated and $Q_{i:\infty}$ the Shifted Q-functions at step $i$. $Q$ is the complete Composite Q-function. Directed incoming edges yield the targets for the corresponding value-function. Edges denoted by $\gamma$ are discounted.}
  \label{fig:composite}
\end{figure}

We estimate the return of $n$-step rollouts of the target-policy via Truncated Q-functions which we then combine to the full return with model-free Shifted Q-functions, an approach we call \textit{Composite Q-learning}, while remaining \textit{purely off-policy}. Since these quantities cannot be estimated directly from single-step transitions, we introduce a consecutive bootstrapping scheme based on intermediate predictions. For an overview, see \Cref{fig:composite}. Code based on the implementation of TD3\footnote{\url{https://github.com/sfujim/TD3}} can be found in the supplementary.

\subsubsection{Truncated Q-functions}
\label{subsubsec:truncq}
In order to decompose the action-value into multiple truncated returns, assume that $n \ll (T-1)$ and that $(T-1-t)\mod n=0$ for task horizon $T$. We make use of the following observation:
\begin{equation}
\label{eq:qsum}
\begin{split}
Q^\pi(s_t, a_t)&=\mathbf{E}_{\pi,\mathcal{M}}\left[r_t+ \gamma r_{t+1} + \gamma^2 r_{t+2} + \gamma^3 r_{t+3} + \dots + \gamma^{T-1} r_{T-1}\right]\\
&=\mathbf{E}_{\pi,\mathcal{M}}\Biggl[\left({ \sum\limits_{j=t}^{t+n-1} \gamma^{j-t} r_j} \right) + \gamma^n \left({\sum\limits_{j=t+n}^{t+2n-1} \gamma^{j-(t+n)} r_j} \right)\\
&\qquad\qquad\qquad\qquad\qquad+ \dots +\gamma^{T-n} \left({ \sum\limits_{j=T-n}^{T-1} \gamma^{j-(T-n)} r_j} \right)\Biggr].
\end{split}
\end{equation}
That is, we can define the action-value as a combination of partial sums of length $n$. We can then define the Truncated Q-function as $Q^\pi_n(s_t, a_t)=\mathbf{E}_{\pi,\mathcal{M}}[\sum_{j=t}^{t+n-1}\gamma^{j-t}r_j]$,
which we plug into \Cref{eq:qsum}:
\begin{equation}\label{eq:trunc}
\begin{split}
Q^\pi(s_t, a_t)&=\mathbf{E}_{\pi,\mathcal{M}}[Q^\pi_n(s_t, a_t)+\gamma^n Q^\pi_n(s_{t+n}, a_{t+n})+\dots+\gamma^{T-n} Q^\pi_n(s_{T-n}, a_{T-n})].\\
\end{split}
\end{equation}

\begin{theorem}
\label{thm:thm1}
Let $Q^\pi_1(s_t, a_t)=r_t$ be the one-step Truncated Q-function and $Q^\pi_{i>1}(s_t, a_t)=r_t + \gamma\mathbf{E}_{t,\pi,\mathcal{M}}[Q^\pi_{i-1}(s_{t+1}, a_{t+1})]$ the i-step Truncated Q-function. Then $Q^\pi_{i}(s_t, a_t)$ represents the truncated return $Q^\pi_i(s_t, a_t)=\mathbf{E}_{t,\pi,\mathcal{M}}[\sum_{j=t}^{t+i-1}\gamma^{j-t}r_j]$.
\end{theorem}

\begin{proof}
  Proof by induction. $Q^\pi_1(s_t, a_t)=r_t$ by definition. The theorem follows from induction step:
  \begin{align*}
    Q^\pi_{i}(s_t, a_t)&=r_t + \gamma\mathbf{E}_{\pi,\mathcal{M}}\left[Q^\pi_{i-1}(s_{t+1}, a_{t+1})\right]\\
    &=r_t + \gamma\mathbf{E}_{\pi,\mathcal{M}}\left[\sum_{j=(t+1)}^{(t+1)+(i-1)-1}\gamma^{j-(t+1)}r_j\right]\\
    &=r_t + \gamma\mathbf{E}_{\pi,\mathcal{M}}\left[\sum_{j=(t+1)}^{t+i-1}\gamma^{j-(t+1)}r_j\right]\\
    &=r_t + \mathbf{E}_{\pi,\mathcal{M}}\left[\sum_{j=(t+1)}^{t+i-1}\gamma^{j-t}r_j\right]\\
    &=\mathbf{E}_{\pi,\mathcal{M}}\left[\sum_{j=t}^{t+i-1}\gamma^{j-t}r_j\right].
  \end{align*}
\end{proof}
Please note, that the assumption of $(T-1-t)\mod n=0$ only leads to an easier notation and is not a general restriction. If $(T-1-t)\mod n\neq0$, then the last partial sum simply has a shorter horizon. Following \Cref{thm:thm1}, we approximate $Q^*_n(s_t, a_t)$ off-policy via consecutive bootstrapping. In order to limit the prediction to horizon $n$, we estimate $n$ different truncated value functions belonging to increasing horizons, with the first one estimating the immediate reward function. All consecutive value functions then bootstrap from the prediction of the preceding value function evaluating the target-policy of \emph{full} return $Q$ for a fixed horizon. Analogous to vanilla Q-learning, the update procedure is given by:
\begin{equation}
  \label{eq:truncatedqlearning}
  \begin{split}
  Q_1(s_t, a_t) &\leftarrow (1-\alpha_\text{Tr})Q_1(s_t, a_t) + \alpha_\text{Tr}r_t \text{ and }\\
  Q_{i>1}(s_t, a_t) &\leftarrow (1-\alpha_\text{Tr})Q_i(s_t, a_t) + \alpha_\text{Tr}(r_t + \gamma Q_{i-1}(s_{t+1}, \argmax_a Q(s_{t+1}, a))),
  \end{split}
\end{equation}
with learning rate $\alpha_\text{Tr}$. Please note that in order to estimate \Cref{eq:trunc}, the dynamics model would be needed to get $s_{t+c\cdot n}$ of a rollout starting in $s_t$. In the following, we describe an approach to achieve an estimation of \Cref{eq:trunc} model-free.

\subsubsection{Shifted Q-functions}
\label{subsubsec:shiftedq}
To get an estimation for the remainder of the rollout $Q_{n:\infty}^\pi=\mathbf{E}_{\pi,\mathcal{M}}[\gamma^n Q(s_{t+n}, a_{t+n})]$ after $n$ steps, we use a consecutive bootstrapping formulation of the Q-prediction as a means to skip the first $n$ rewards of a target-policy rollout.

\begin{theorem}
  \label{thm:thm2}
Let $Q^\pi_{1:\infty}(s_t, a_t)=\mathbf{E}_{t,\pi,\mathcal{M}}[\gamma Q^\pi(s_{t+1}, a_{t+1})]$ be the one-step Shifted Q-function and $Q^\pi_{i>1:\infty}(s_t, a_t)=\mathbf{E}_{t,\pi,\mathcal{M}}[\gamma Q^\pi_{i-1:\infty}(s_{t+1}, a_{t+1})]$ the i-step Shifted Q-function. Then $Q^\pi_{i:\infty}(s_t, a_t)$ represents the shifted return $Q^\pi_{i:\infty}(s_t, a_t)=\mathbf{E}_{t,\pi,\mathcal{M}}[\gamma^i Q^\pi(s_{t+i}, a_{t+i})]$.
\end{theorem}

\begin{proof}
  Proof by induction. $Q^\pi_{1:\infty}(s_t, a_t)=\mathbf{E}_{\pi,\mathcal{M}}[\gamma Q^\pi(s_{t+1}, a_{t+1})]$ by definition. The theorem follows from induction step:
  \begin{align*}
    Q^\pi_{i:\infty}(s_t, a_t)&=\mathbf{E}_{\pi,\mathcal{M}}\left[\gamma Q^\pi_{i-1:\infty}(s_{t+1}, a_{t+1})\right]\\
    &=\mathbf{E}_{\pi,\mathcal{M}}\left[\gamma (\gamma^{i-1}Q^\pi(s_{t+1+i-1}, a_{t+1+i-1}))\right]\\
    &=\mathbf{E}_{\pi,\mathcal{M}}\left[\gamma (\gamma^{i-1}Q^\pi(s_{t+i}, a_{t+i}))\right]\\
    &=\mathbf{E}_{\pi,\mathcal{M}}\left[\gamma^{i}Q^\pi(s_{t+i}, a_{t+i})\right].
  \end{align*}
\end{proof}
In accordance to the consecutive bootstrapping formulation of Truncated Q-functions, the updates for the Shifted Q-functions are:
\begin{equation}
\label{eq:shiftedupdate}
  \begin{split}
  Q_{1:\infty}(s_t, a_t) &\leftarrow (1-\alpha_\text{Sh})Q_{1:\infty}(s_t, a_t) + \alpha_\text{Sh}(\gamma \max_a Q(s_{t+1}, a)) \text{ and}\\
  Q_{(i>1):\infty}(s_t, a_t) &\leftarrow (1-\alpha_\text{Sh}) Q_{i:\infty}(s_t, a_t) + \alpha_\text{Sh}(\gamma Q_{(i-1):\infty}(s_{t+1}, \argmax_a Q(s_{t+1}, a))),
  \end{split}
\end{equation}
with learning rate $\alpha_\text{Sh}$. We hypothesize that Shifted Q-functions allow for higher learning rates, since they only have to account for the variance of the underlying dynamics and not of the immediate reward function, as opposed to conventional value functions.
\subsubsection{Composition}
\label{sec:Composition}
\begin{algorithm}[h]
  \small
  initialize Truncated Q-functions $Q_i$\\
  initialize Shifted Q-functions $Q_{i:\infty}$\\
  initialize Q-function $Q$\\
  \For{$episode=1..E$} {
    get initial state $s_1$\\
    \For{$t=1..T$}{
      apply $\epsilon$-greedy action $a_t$\\
      observe $s_{t+1}$ and $r_t$\\
      update Truncated Q-functions by:\\
        \nonl\hspace{0.19cm}$Q_1(s_t, a_t) \leftarrow (1-\alpha_\text{Tr})Q_1(s_t, a_t) + \alpha_\text{Tr}r_t$\\
        \nonl\hspace{0.19cm}$Q_{i>1}(s_t, a_t) \leftarrow (1-\alpha_\text{Tr})Q_i(s_t, a_t) + \alpha_\text{Tr}(r_t + \gamma Q_{i-1}(s_{t+1}, \argmax_a Q(s_{t+1}, a)))$\\
      update Shifted Q-functions by:\\
  \nonl\hspace{0.19cm}$Q_{1:\infty}(s_t, a_t) \leftarrow (1-\alpha_\text{Sh})Q_{1:\infty}(s_t, a_t) + \alpha_\text{Sh}(\gamma \max_a Q(s_{t+1}, a))$\\
  \nonl\hspace{0.19cm}$Q_{(i>1):\infty}(s_t, a_t) \leftarrow (1-\alpha_\text{Sh}) Q_{i:\infty}(s_t, a_t) + \alpha_\text{Sh}(\gamma Q_{(i-1):\infty}(s_{t+1}, \argmax_a Q(s_{t+1}, a)))$\\
      update Q-function by:\\
      \nonl\hspace{0.19cm}$Q(s_t, a_t) \leftarrow (1-\alpha_Q)Q(s_t, a_t) + \alpha_Q(r_t + \gamma (Q_n(s_{t+1}, \argmax_a Q(s_{t+1}, a))$\\
      \nonl\hspace{6.75cm}$+ Q_{n:\infty}(s_{t+1}, \argmax_a Q(s_{t+1}, a))))$
    }
  }
\caption{Composite Q-learning}
\label{alg:compositeqlearning}
\end{algorithm}

Following the definitions of Truncated and Shifted Q-functions, we can compose the full return.
\begin{theorem}
  \label{thm:thm3}
Let $Q_n^\pi(s_t,a_t)=\mathbf{E}_{t,\pi,\mathcal{M}}[\sum_{j=t}^{t+n-1}\gamma^{j-t}r_j]$ be the truncated return and $Q_{n:\infty}^\pi(s_t,a_t)=\mathbf{E}_{t,\pi,\mathcal{M}}[\gamma^n Q(s_{t+n}, a_{t+n})]$ the shifted return. Then $Q^\pi(s_t, a_t)=Q_n^\pi(s_t,a_t)+Q_{n:\infty}^\pi(s_t,a_t)$ represents the full return, i.e. $Q^\pi(s_t, a_t)=\mathbf{E}_{t,\pi,\mathcal{M}}[\sum_{j=t}^\infty\gamma^{j-t}r_j]$.
\end{theorem}

\begin{proof}
  \begin{align*}
    Q^\pi(s_t,a_t) &= Q_n^\pi(s_t,a_t)+Q_{n:\infty}^\pi(s_t,a_t)\\
    &= \mathbf{E}_{\pi,\mathcal{M}}\left[\sum\limits_{j=t}^{t+n-1}\gamma^{j-t}r_j + \gamma^n Q^\pi(s_{t+n}, a_{t+n})\right]\\
    &= \mathbf{E}_{\pi,\mathcal{M}}\left[\sum\limits_{j=t}^{t+n-1}\gamma^{j-t}r_j + \gamma^n  \Biggl(Q_n^\pi(s_{t+n},a_{t+n})+Q_{n:\infty}^\pi(s_{t+n},a_{t+n})\Biggr)\right]\\
    &= \mathbf{E}_{\pi,\mathcal{M}}\left[\sum\limits_{j=t}^{t+n-1}\gamma^{j-t}r_j + \gamma^n  \Biggl(\sum\limits_{j=t+n}^{t+2n-1}\gamma^{j-t-n}r_j+\gamma^n Q^\pi(s_{t+2n}, a_{t+2n})\Biggr)\right]\\
    &= \mathbf{E}_{\pi,\mathcal{M}}\left[\sum\limits_{j=t}^{t+n-1}\gamma^{j-t}r_j + \sum\limits_{j=t+n}^{t+2n-1}\gamma^{j-t}r_j+\gamma^{2n} Q^\pi(s_{t+2n}, a_{t+2n})\right]\\
    &= \mathbf{E}_{\pi,\mathcal{M}}\left[\sum\limits_{j=t}^{t+2n-1}\gamma^{j-t}r_j+\gamma^{2n} Q^\pi(s_{t+2n}, a_{t+2n})\right].
  \end{align*}
  By induction, it follows $Q^\pi(s_t, a_t)=\mathbf{E}_{\pi,\mathcal{M}}[\sum_{j=t}^\infty\gamma^{j-t}r_j]$.
\end{proof}
Thus, we can formalize the update of the full Q-function by:
\begin{equation}
  \begin{split}
  Q(s_t, a_t) \leftarrow (1-\alpha_Q)Q(s_t, a_t) + \alpha_Q(r_t + \gamma (&Q_n(s_{t+1}, \argmax_a Q(s_{t+1}, a))\\
                                                          &+ Q_{n:\infty}(s_{t+1}, \argmax_a Q(s_{t+1}, a)))).
  \end{split}
\end{equation}
The incorporation of truncated returns breaks down the time scale of the long-term prediction by the Shifted Q-function. We call this algorithm \textit{Composite Q-learning} (cf. \Cref{alg:compositeqlearning}). Please note that Composite Q-learning is equivalent to Q-learning with learning rate $\alpha_Q$ when setting $\alpha_\text{Tr}=\alpha_\text{Sh}=\alpha_Q$. However, a notable advantage is that Composite Q-learning offers an independent optimization for the different learning rates corresponding to different temporal horizons.

\subsection{Deep Composite Q-learning}
In order to cope with infinite or continuous state and action-spaces, we extend Composite Q-learning to the function approximation setting.
\subsubsection{Target Formulation}
Let $Q^{\text{Tr}}(\cdot, \cdot|\theta^{{\text{Tr}}})$ denote a function approximator with parameters $\theta^{{\text{Tr}}}$ and $n$ outputs, subsequently called \textit{heads}, estimating Truncated Q-functions $Q_i$. Each output $Q^{\text{Tr}}_i$ bootstraps from the prediction of the preceding head, with the first approximating the immediate reward function. The targets are therefore given by:
\begin{equation}\label{eq:trunctarget0}
  y^{\text{Tr}}_{j, 1} = r_j \text{ and } y^{\text{Tr}}_{j, i>1} = r_j + \gamma Q_{i-1}^{\text{Tr}\prime}(s_{j+1}, \mu^{\prime}(s_{j+1}|\theta^{\mu\prime})|\theta^{{\text{Tr}\prime}}),
\end{equation}
where $\mu^{\prime}$ corresponds to the target-actor maximizing the full Q-value as defined in \Cref{sec:background} and $Q_{i-1}^{\text{Tr}\prime}$ the output of the respective Q-target-network. That is, $Q^{\text{Tr}}$ represents evaluations of $\mu$ at different stages of truncation and $y^{\text{Tr}}_{j, i<n}$ serve as intermediate predictions to get $y^{\text{Tr}}_{j, n}$. We then \textit{only} use $Q^\text{Tr}_n$, which implements the full $n$-step return, as the first part of the composition of the Q-target, but not the $Q^\text{Tr}_{i<n}$ values.\\

The second part of the composition is represented by the Shifted Q-function. Let $Q^{\text{Sh}}(\cdot, \cdot|\theta^{{\text{Sh}}})$ denote the function approximator estimating the Shifted Q-functions $Q_{i:\infty}^\pi$ by $n$ different outputs, parameterized by $\theta^{{\text{Sh}}}$. We can shift the Q-prediction by bootstrapping without taking the immediate reward into account, so as to skip the first $n$ rewards of a target-policy rollout. The Shifted Q-targets for heads $Q^{\text{Sh}}_i$ therefore become:
\begin{equation}\label{eq:shiftedtarget0}
y_{j, 1}^\text{Sh}=\gamma Q^\prime(s_{j+1}, \mu^\prime(s_{j+1}|\theta^{\mu\prime})|\theta^{Q\prime}) \text{ and } y_{j, i>1}^\text{Sh}=\gamma Q^{\text{Sh}\prime}_{i-1}(s_{j+1}, \mu^\prime(s_{j+1}|\theta^{\mu\prime})|\theta^{{\text{Sh}\prime}}).
\end{equation}  
In the function approximation setting, we can thus define the Composite Q-target as:
\begin{equation}\label{eq:qtarget}
y^Q_j=r_j + \gamma (Q^{\text{Tr}\prime}_n(s_{j+1}, \mu^\prime(s_{j+1}|\theta^{\mu\prime})|\theta^{{\text{Tr}\prime}}) + Q^{\text{Sh}\prime}_n(s_{j+1}, \mu^\prime(s_{j+1}|\theta^{\mu\prime})|\theta^{{\text{Sh}\prime}})),
\end{equation}
approximated by $Q(\cdot, \cdot|\theta^Q)$ with parameters $\theta^Q$. Since we have true reward $r_j$, we include it in the target. Please note, that shifting the action-value in time imposes a bottleneck, since it relies on the estimation of the full action-value. We therefore jointly estimate $Q^{\text{Tr}}$, $Q^{\text{Sh}}$ and $Q$ with function approximator $Q^\mathcal{C}(\cdot, \cdot|\theta^{\mathcal{C}})$.

\begin{algorithm}
  \small
  initialize critic $Q^{\mathcal{C}}$, actor $\mu$ and targets $Q^{{\mathcal{C}}\prime}$, $\mu'$\\
  initialize replay buffer $\mathcal{R}$\\
  \For{$episode=1..E$} {
    get initial state $s_1$\\
    \For{$t=1..T$}{
      apply action $a_t=\mu(s_t|\theta^\mu) + \xi$, where $\xi\sim\mathcal{N}(0, \sigma)$\\
      observe $s_{t+1}$ and $r_t$ and save transition $(s_t,a_t,s_{t+1},r_t)$ in $\mathcal{R}$\\
      calculate targets:\\\label{line:targets}
      {
      \nonl\qquad$y^{\text{Tr}}_{j, 1} = r_j$\\
      \nonl\qquad$y^{\text{Tr}}_{j, i>1} = r_j + \gamma Q_{i-1}^{\text{Tr}\prime}(s_{j+1}, \mu^{\prime}(s_{j+1}|\theta^{\mu\prime})|\theta^{{\text{Tr}\prime}})$\\
      \nonl\qquad$y_{j, 1}^\text{Sh}=\gamma Q^\prime(s_{j+1}, \mu^\prime(s_{j+1}|\theta^{\mu\prime})|\theta^{Q\prime})$\\
      \nonl\qquad$y_{j, i>1}^\text{Sh}=\gamma Q^{\text{Sh}\prime}_{i-1}(s_{j+1}, \mu^\prime(s_{j+1}|\theta^{\mu\prime})|\theta^{{\text{Sh}\prime}})$\\
      \nonl\qquad$y^Q_j=r_j + \gamma (Q^{\text{Tr}\prime}_n(s_{j+1}, \mu^\prime(s_{j+1}|\theta^{\mu\prime})|\theta^{{\text{Tr}\prime}}) + Q^{\text{Sh}\prime}_n(s_{j+1}, \mu^\prime(s_{j+1}|\theta^{\mu\prime})|\theta^{{\text{Sh}\prime}}))$\\

      \nonl\qquad $y^\mathcal{C}_j=\left[y^Q_j,y^{\text{Tr}}_{j, 1},y^{\text{Tr}}_{j, i>1},y_{j, 1}^\text{Sh},y_{j, i>1}^\text{Sh}\right]$}\\
      
      update $Q^\mathcal{C}$ on minibatch $b$ of size $m$ from $\mathcal{R}$ according to \Cref{eq:loss}\\ 
      update $\mu$ on $Q$\\
      adjust parameters of $Q^{\mathcal{C}\prime}$ and $\mu'$\\
    }
  }
\caption{Deep Continuous Composite Q-learning}
\label{alg:compositeddpg}
\end{algorithm}

\subsubsection{Entropy Regularization}
\label{subsubsec:ent}
Each pair $Q^{\text{Tr}}_i + Q^{\text{Sh}}_i|_{1\leq i\leq n}$ is a complete approximation of the true Q-value. Note however, that it is also bootstrapped with the full Q-value (cf. \Cref{fig:composite} and \Cref{eq:shiftedupdate}). The circular dependency can lead to stability issues, due to the amplification of propagated errors. Additionally, higher learning rates for the Shifted Value-functions, as motivated in \Cref{subsubsec:shiftedq}, may lead to overfitting. In particular, errors in the truncated predictions are propagated quickly by the Shifted Q-functions, especially when trained with a high learning rate. Therefore, it is important to keep the truncated predictions close across the different time steps while also preventing the Shifted Q-functions from running into suboptimal local minima. As a means to alleviate these issues, we add a novel regularization technique based on the entropy of a Gaussian distribution formed by mean $\mu(Q^{\text{Tr}}_i + Q^{\text{Sh}}_i|_{1\leq i\leq n})$ and variance $\sigma^2(Q^{\text{Tr}}_i + Q^{\text{Sh}}_i|_{1\leq i\leq n})$ over all $n$ complete predictions of $Q$. We can then formalize an incentive of the Truncated Q-functions to stay close between predictions whilst forcing the Shifted Q-functions to keep the distribution as broad as possible. We define the Gaussian distribution $\mathcal{N}_Q$ over all Q-predictions therefore by:

\begin{equation}\mathcal{N}_Q\left(\mu(Q^{\text{Tr}}_i + Q^{\text{Sh}}_i|_{1\leq i\leq n}), \sigma^2(Q^{\text{Tr}}_i + Q^{\text{Sh}}_i|_{1\leq i\leq n})\right),\end{equation}
with corresponding entropy: \begin{equation}\mathcal{H}_Q=\frac{1}{2}\log{\left(2\pi e\sigma^2(Q^{\text{Tr}}_i + Q^{\text{Sh}}_i|_{1\leq i\leq n})\right)}.\end{equation}
In order to enhance stability of the learning process as a whole while preventing the Shifted Q-functions from overfitting, we not only minimize the mean squared error between targets $y^\mathcal{C}_j$ and predictions $Q^\mathcal{C}(s_j,a_j|\theta^{\mathcal{C}})$, but apply gradient descent on the entropy for the parameters of $Q^\text{Tr}$ and gradient ascent on the parameters of $Q^\text{Sh}$.
More detailed, we define the squared error for some sample $j$ as:
\begin{equation}\delta_j = \left(y^\mathcal{C}_j - Q^\mathcal{C}(s_j, a_j|\theta^{\mathcal{C}})\right)^2,\end{equation}
and the gradient w.r.t. all parameters of $Q^\mathcal{C}$ including the different learning rates as:
\begin{equation}\xi_j=\alpha_Q\delta_j\nabla_{\theta^{Q}}Q^\mathcal{C}(s_j, a_j|\theta^{\mathcal{C}})+\alpha_\text{Tr}\delta_j\nabla_{\theta^{\text{Tr}}}Q^\mathcal{C}(s_j, a_j|\theta^{\mathcal{C}})+\alpha_\text{Sh}\delta_j\nabla_{\theta^{\text{Sh}}}Q^\mathcal{C}(s_j, a_j|\theta^{\mathcal{C}}).\end{equation}
The regularization then adds to the gradient:
\begin{equation}\eta_j=\mathcal{H}(s_j, a_j)(\alpha_\text{Sh}\beta_\text{Sh}\nabla_{\theta^{\text{Sh}}}-\alpha_\text{Tr}\beta_\text{Tr}\nabla_{\theta^{\text{Tr}}}).\end{equation}
Hence, the parameter update in its simplest form becomes:
\begin{equation}\label{eq:loss}\theta^{\mathcal{C}} \leftarrow \theta^{\mathcal{C}} + \frac{1}{m}\sum_j(\xi_j - \eta_j).\end{equation}
However, a more sophisticated optimizer such as Adam \citep{adam} can be applied analogously. A detailed description of Deep Continuous Composite Q-learning is given in \Cref{alg:compositeddpg}, where the adjustments of TD3 are omitted for simplicity. In order to transform \Cref{alg:compositeddpg} to its TD3-equivalent, Gaussian policy smoothing has to be added to all targets in \Cref{line:targets}, as well as taking the minimum prediction of two distinct critics for each target. Furthermore, actor and target networks have to be updated with delay.
\subsection{TD3(\texorpdfstring{$\Delta$}{-delta})}
\label{subsec:mbvdelta}
Another way to divide the value function into multiple time scales is TD($\Delta$) \citep{timescales}. To this point, it has only been applied in an on-policy setting. In favor of comparability, we extend TD($\Delta$) to Q-learning, yielding TD3($\Delta$). The main idea of TD($\Delta$) is the combination of different value functions corresponding to increasing discount values. Let $\gamma^\Delta$ denote a fixed ordered sequence of increasing discount values, i.e. $\gamma^\Delta=(\gamma_1, \gamma_2, \dots, \gamma_k)^\top|_{\gamma_{i>1} >\gamma_{i-1}}$.
\begin{algorithm}
  \small
  initialize critic $Q^{\Delta}$, actor $\mu$ and targets $Q^{\Delta\prime}$, $\mu^\prime$\\
  initialize replay buffer $\mathcal{R}$\\
  set discount values $\gamma^\Delta=(\gamma_0, \gamma_1, \dots, \gamma_k)^\intercal$\\
  \For{$episode=1..E$} {
    get initial state $s_1$\\
    \For{$t=1..T$}{
      apply action $a_t=\mu(s_t|\theta^\mu) + \xi$, where $\xi\sim\mathcal{N}(0, \sigma)$\\
      observe $s_{t+1}$ and $r_t$ and save transition $(s_t,a_t,s_{t+1},r_t)$ in $\mathcal{R}$\\
      calculate targets:\\
      \nonl\qquad$y^\gamma_{j,1} = r_j + \gamma_1 Q^{\prime}_{\gamma_1}(s_{j+1}, \mu^\prime(s_{j+1}|\theta^{\mu\prime})|\theta^{\Delta{\prime}})$\\
      \nonl\qquad$y^\gamma_{j, i>1} = (\gamma_i-\gamma_{i-1}) Q^\prime_{\gamma_{i-1}}(s_{j+1}, \mu^\prime(s_{j+1}|\theta^{\mu\prime})|\theta^{\Delta\prime})$\\
      \nonl\qquad\qquad\qquad$+\gamma_i W^\prime_{i}(s_{j+1}, \mu^\prime(s_{j+1}|\theta^{\mu\prime})|\theta^{\Delta\prime})$\\
      update $Q^{\Delta}$ on minibatch $b$ of size $m$ from $\mathcal{R}$\\
      update $\mu$ on $Q_{\gamma_k}$\\
      adjust parameters of $Q^{\Delta\prime}$ and $\mu^\prime$\\
    }
  }
\caption{Off-policy TD($\Delta$)}
\label{alg:ddpgdelta}
\end{algorithm}

We can then define delta functions $W_i$ as: \begin{equation}W_1 = Q_{\gamma_1} \text{ and } W_{i>1} = Q_{\gamma_i} - Q_{\gamma_{i-1}}.\end{equation}

Let $Q^\Delta(\cdot, \cdot|\theta^{\Delta})$ denote the function approximator estimating $Q_{\gamma_{1\leq i\leq k}}$ with parameters $\theta^{\Delta}$. Based on the derivations in~\citep{timescales}, the targets for Q-learning can be formalized as:
\begin{equation}
\begin{gathered}
  y^\gamma_{j,1} = r_j + \gamma_1 Q^\prime_{\gamma_1}(s_{j+1}, \mu^\prime(s_{j+1}|\theta^{\mu\prime})|\theta^{\Delta\prime}) \text{ and}\\
  y^\gamma_{j,i>1} = (\gamma_i - \gamma_{i-1}) Q^\prime_{\gamma_{i-1}}(s_{j+1}, \mu^\prime(s_{j+1}|\theta^{\mu\prime})|\theta^{\Delta\prime}) + \gamma_i W^\prime_i(s_{j+1}, \mu^\prime(s_{j+1}|\theta^{\mu\prime})|\theta^{\Delta\prime}).
\end{gathered}
\end{equation}
The authors suggest the use of $n$-step targets within TD($\Delta$) which is not easily applicable in an off-policy setting. In our experiments, we therefore compare our approach to single-step TD3$(\Delta)$. A detailed description of Off-policy TD($\Delta$) is given in \Cref{alg:ddpgdelta}. To transform Off-policy TD($\Delta$) to TD3($\Delta$), the adjustments as described in \Cref{subsubsec:ent} have to be applied analogously.

\section{Experimental Results}
\label{sec:experiments}
In this section, we evaluate Tabular and Deep Composite Q-learning. In the tabular case, we apply Composite Q-learning on deterministic and stochastic chain MDPs of varying horizon to illustrate and analyze the characteristics of the interplay between short-term and long-term predictions. We then discuss advantages of Deep Composite Q-learning under a noise-free reward for the two robot simulation tasks Walker2d-v2 and Hopper-v2 and evaluate Deep Composite Q-learning for Walker2d-v2, Hopper-v2 and Humanoid-v2 under a noisy reward function, comparing it to TD3 and TD3($\Delta$).

\subsection{Tabular Composite Q-learning}
\label{subsec:tabularq}
\subsubsection{Deterministic Chain}
We start out with an evaluation of the effect of composing the long-term return of multiple short-term predictions. To this end, we apply Composite Q-learning in the tabular case to the MDP of horizon $K$ given in \Cref{fig:motivmdp}. We compare Composite Q-learning to vanilla Q-learning, as well as multi-step Q-learning based on subtrajectories of the exploratory policy and imaginary rollouts of the target-policy with the true model of the MDP as a hypothetical lower bound on the required updates.\\
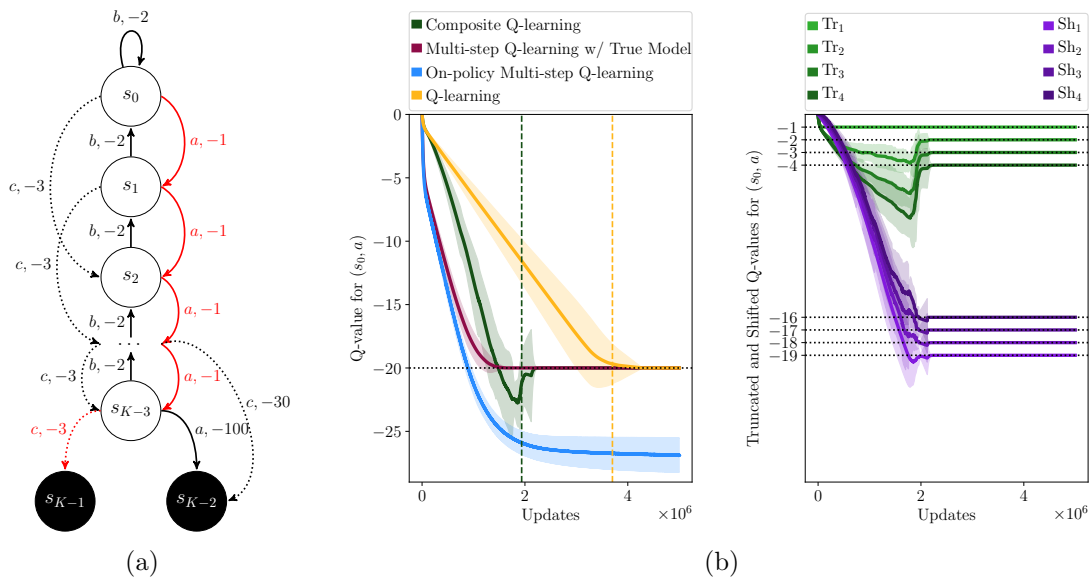
\begin{figure}[h]
  \begin{subfigure}{0.25\textwidth}
    \centering
    \resizebox{!}{7cm}{
\begin{tikzpicture}[->,>=stealth',shorten >=1pt,auto,node distance=2.8cm,
                    semithick]
  \tikzstyle{every state}=[]

  \node[state] (0) [minimum height=1.5cm]                    {\Large$s_0$};
  \node[state] (1) [below=0.75cm of 0, minimum height=1.5cm] {\Large$s_1$};
  \node[state] (2) [below=0.75cm of 1, minimum height=1.5cm] {\Large$s_2$};
  \node        (3) [below=0.75cm of 2, minimum width=1.5cm] {\Large$\dots$};
  \node[state] (4) [below=0.75cm of 3, minimum height=1.5cm] {\Large$s_{K-3}$};
  
  \node (mid) [below=0.75cm of 4, minimum height=1.5cm] {};
  
  \node[state] (5) [right=0.75cm of mid, minimum height=1.5cm, fill=black, text=white] {\Large$s_{K-2}$};
  \node[state] (6) [left=0.75cm of mid, minimum height=1.5cm, fill=black, text=white] {\Large$s_{K-1}$};

  \path (0.east) edge [right, very thick, out=330, in=30, red] node {\large$a, -1$}  (1.east)
        (1.east) edge [right, very thick, out=330, in=30, red] node {\large$a, -1$}  (2.east)
        (2.east) edge [right, very thick, out=330, in=30, red] node {\large$a, -1$}  (3.east)
        (3.east) edge [right, very thick, out=330, in=30, red] node {\large$a, -1$}  (4.east)
        (4.east) edge [right, very thick, out=0, in=90] node {\large$a, -100$}  (5.north)
        (4.north) edge [left, very thick] node {\large$b, -2$}  (3.south)
        (3.north) edge [left, very thick] node {\large$b, -2$}  (2.south)
        (2.north) edge [left, very thick] node {\large$b, -2$}  (1.south)
        (1.north) edge [left, very thick] node {\large$b, -2$}  (0.south)
        (0) edge [loop above, very thick] node {\large$b, -2$} (0)
        (0.west) edge [left, very thick, out=200, in=160, dotted] node {\large$c, -3$}  (2.west)
        (1.west) edge [left, very thick, out=200, in=160, dotted] node {\large$c, -3$}  (3.west)
        (3.west) edge [left, very thick, out=200, in=160, dotted] node {\large$c, -3$}  (4.west)
        (3.east) edge [right, very thick, out=0, in=0, dotted, bend left=90] node {\large$c, -30$}  (5.east)
        (4.west) edge [left, very thick, out=180, in=90, dotted, red] node {\large$c, -3$}  (6.north);

\end{tikzpicture}
    }
    \caption{}
    \label{fig:motivmdp}
  \end{subfigure}
  \begin{subfigure}{0.75\textwidth}
    \centering
    \includegraphics[height=7cm]{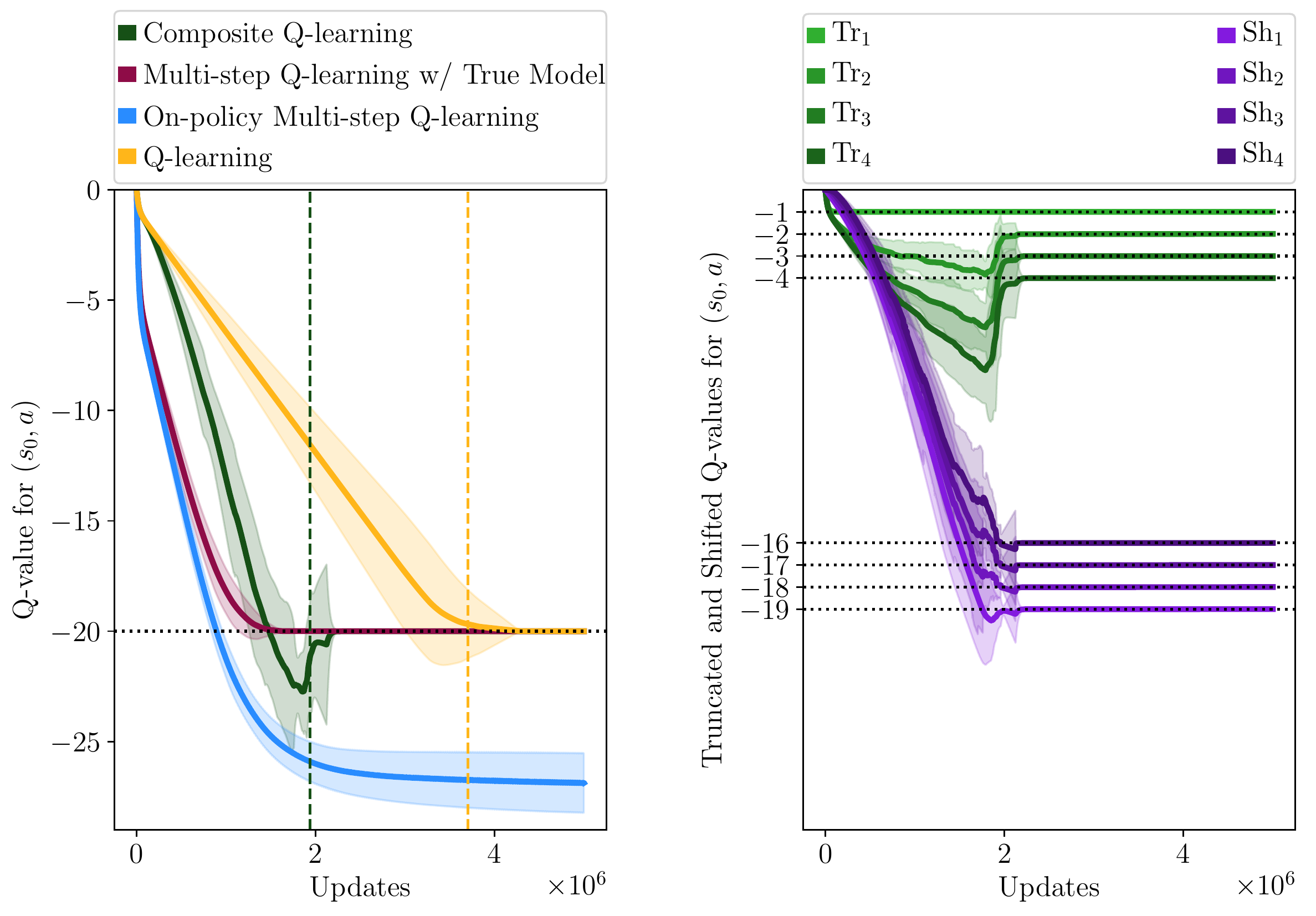}
    \caption{}
    \label{fig:motivmdpres}
  \end{subfigure}
  \caption{(a) In this deterministic-chain MDP of horizon $K$, the agent ought to arrive at terminal state $s_{K-1}$ using actions $\{a,b,c\}$. The initial state is $s_0$ and the optimal policy is given in red. (b) Mean results and two standard deviations over 10 runs on the MDP with a horizon of $K=20$. The left plot depicts the value of $s_0$ and action $a$ as estimated by the different approaches over time. Dashed lines indicate convergence to the optimal policy. The predicted Truncated Q-values for state $s_0$ and action $a$ with horizons 1 to 4, denoted by $\text{Tr}_1, \dots, \text{Tr}_4$, and predicted Shifted Q-values for state $s_0$ and action $a$, denoted by $\text{Sh}_1, \dots, \text{Sh}_4$, are to the right. Dotted lines indicate the true optimal respective Q-values.}
  
  \label{fig:motiv}
\end{figure}

Results for $K=20$ are depicted in \Cref{fig:motivmdpres}, where mean performances of the approaches are shown on the left and the intermediate predictions of Composite Q-learning on the right. All approaches update the Q-function with a learning rate of $10^{-3}$ on the same fixed batch of $10^3$ episodes with a percentage of $10\%$ non-optimal transitions. We set rollout
\begin{wrapfigure}{l}{0.4\textwidth}
    \centering
    \includegraphics[width=0.4\textwidth]{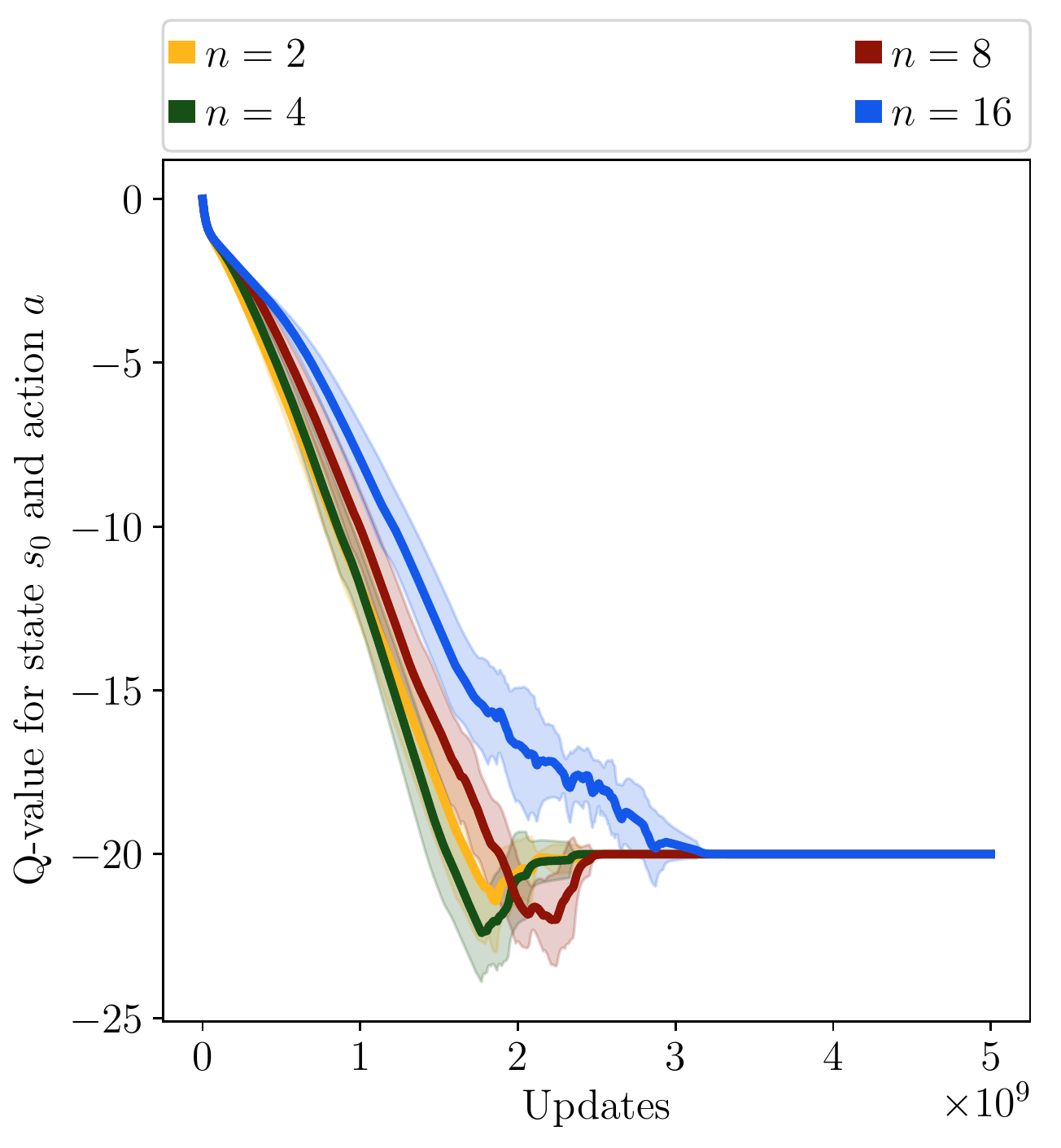}
    \caption{Performance of Composite Q-learning applied to the MDP in \Cref{fig:motivmdp} with different rollout lengths $n$.}
    \label{fig:rolloutsdet}
\end{wrapfigure}length $n=4$ for all respective approaches. Please note that this value can be set arbitrarily, however in order to be beneficial it should be set to a value smaller than the temporal horizon of the task. In this example, a rollout length of 4 corresponds to the integer square root of the horizon of the MDP, meaning that Truncated and Shifted Q-functions have a similar temporal horizon due to the partitioning of the long-term value. An evaluation of the performance for different rollout lengths $n$ is shown in \Cref{fig:rolloutsdet}. For the experiments in \Cref{fig:motivmdpres}, we update the Shifted Q-function with a learning rate of $10^{-2}$ and the Truncated Q-functions with a learning rate of $10^{-3}$. Composite Q-learning converges much faster to the true action-value than vanilla Q-learning and is close to the multi-step approach based on model rollouts (hypothetical lower bound). The erroneous updates of on-policy multi-step Q-learning lead to convergence to a wrong action-value which is underlining the importance of truly off-policy learning. Composite Q-learning estimates the short-term rollouts off-policy by design and shifts the long-term value in time so as to have a consistent temporal chaining. All intermediate predictions therefore converge to the true intermediate action-values, which is shown on the right side of \Cref{fig:motivmdpres}. The difference in convergence speed between Q-learning and Composite Q-learning grows with increasing horizon, as shown in \Cref{tab:horiz} for the described hyperparameter setting.\\
\begin{table}[h]
      \caption{Comparison of convergence speed between Tabular Q-learning and Tabular Composite Q-learning for exemplary runs on the MDP given in \Cref{fig:motivmdp} with $n=4$.}
      \small
      \centering
      \begin{tabular}{lcccc}
        \toprule
        Horizon $K$&10&20&50&100\\
        \midrule
        Speed up over Q-learning&$11\%$&$44\%$&$57\%$&$66\%$\\
        \bottomrule
      \end{tabular}
      \label{tab:horiz}
    \end{table}

Next, we present an evaluation of different learning rates for the Shifted Q-functions, depicted in \Cref{fig:motivmdpresshifted}. Here we compare Composite Q-learning and vanilla Q-learning to Shifted Q-learning, which corresponds to Q-learning with a one-step shifted target (without approximate $n$-step returns from a Truncated Q-function). The results show that shifting the value in time alone is slowing down convergence. Precisely, Shifted Q-learning with a learning rate of $1.0$ is equivalent to vanilla Q-learning; the same holds for Composite Q-learning with a learning rate of $10^{-3}$ for the Shifted Q-function. The results underline that the higher learning rates for the Shifted Q-function only have a beneficial effect in combination with truncated predictions.\\

 \begin{figure}[t]
  \centering
  \begin{subfigure}{0.49\textwidth}
  \includegraphics[width=\textwidth]{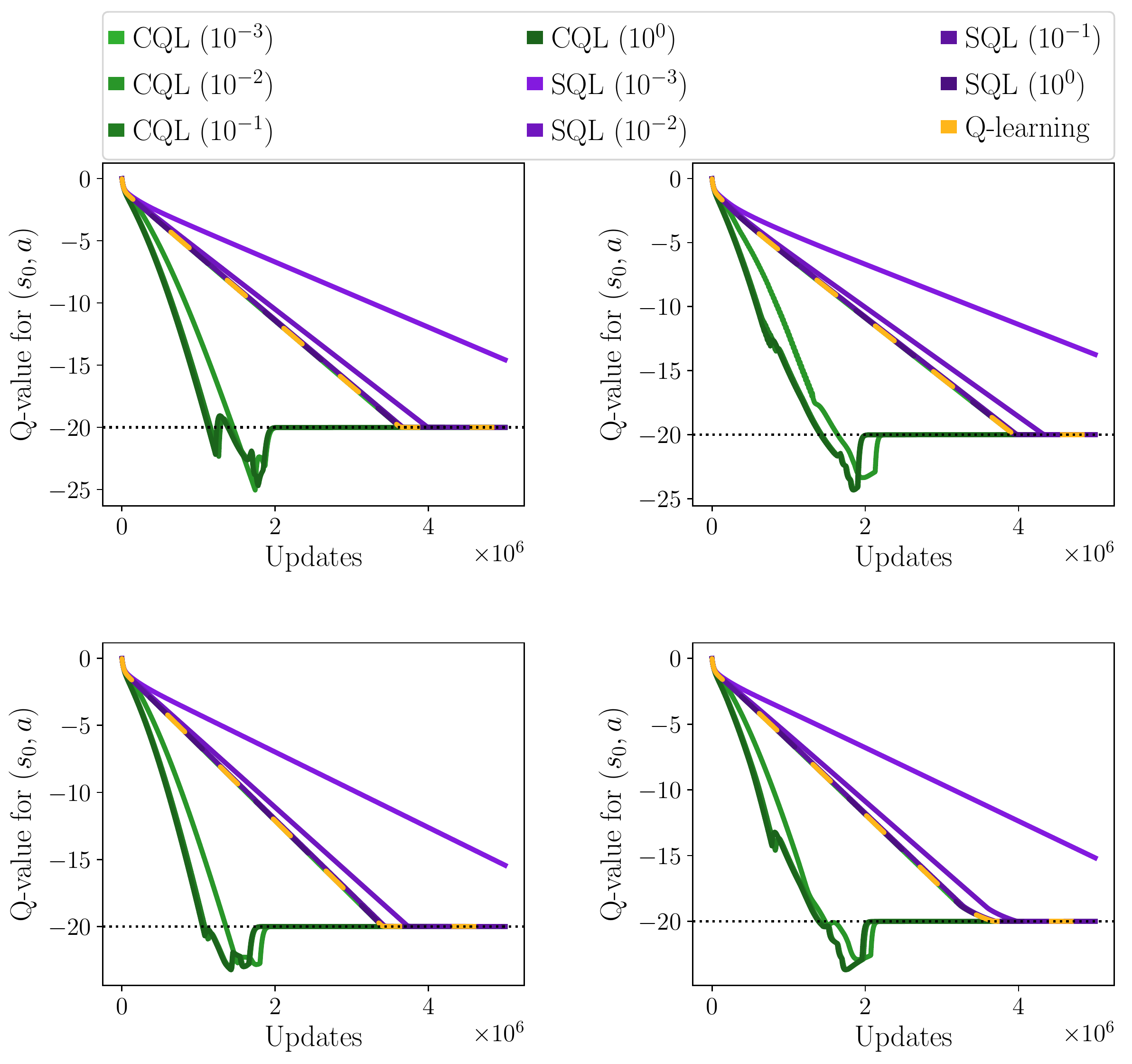}
  \caption{}
  \label{fig:motivmdpresshifted}
  \end{subfigure}
  \begin{subfigure}{0.49\textwidth}
  \includegraphics[width=\textwidth]{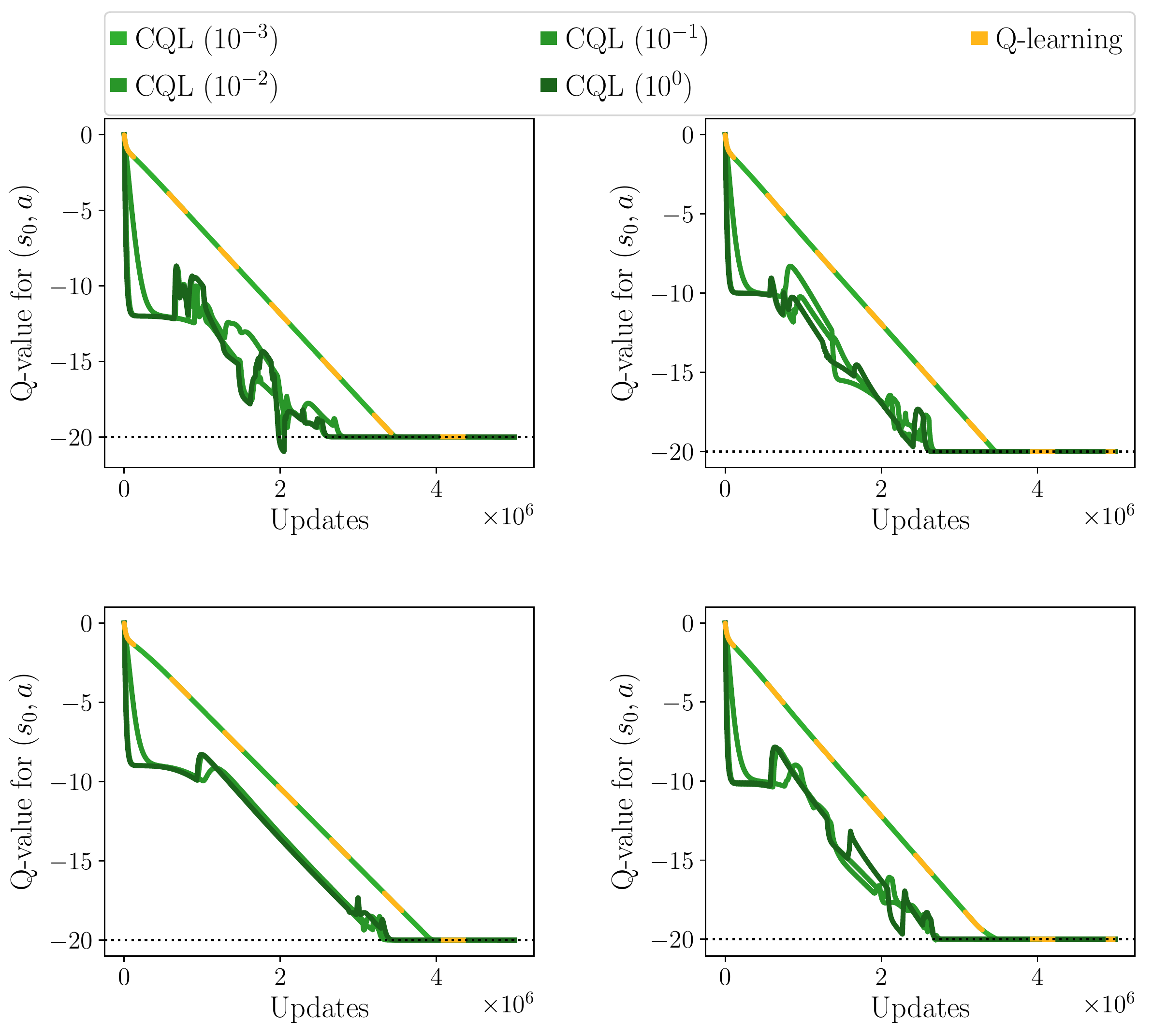}
  \caption{}
  \label{fig:motivmdprestruncated}
  \end{subfigure}
  \caption{Results of 4 individual runs on the deterministic chain MDP with a horizon of $K=20$ for Composite and Shifted Q-learning, (a) with different learning rates for the Shifted Q-function and (b) with different learning rates for the Truncated Q-function (denoted by the numbers in parentheses). The learning rates for the full Q-function and for the (a) Truncated Q-functions and the (b) Shifted Q-functions are set to $10^{-3}$ in all experiments.}
\end{figure}

The counterpart with different learning rates for the Truncated Q-functions,  keeping the learning rates for the full Q-estimate and Shifted Q-functions fixed, can be seen in \Cref{fig:motivmdprestruncated}. While there is improvement in convergence using a larger learning rate for the Truncated Q-functions, the results show higher variance and less benefit than the Shifted Q-functions in \Cref{fig:motivmdpresshifted}. The results suggest that the interplay between short- and long-term predictions yields the most benefit if the learning rate for the Shifted Q-functions can be set to a higher value. We believe this to be due to two possible reasons. The higher learning rates for the Truncated Q-functions might lead to overfitting on shorter horizons more quickly (given that it represents an easier learning problem compared to the full return), yet it is very important for those values to be of small error, since they build the basis for the temporal chaining completed by the Shifted Q-functions. On the other hand, the Shifted Q-functions allow for higher learning rates, since they only have to consider the variability of the distribution over next states, not immediate rewards, leading to decreased variance in their targets.

\subsubsection{Stochastic Chain}
We further investigate the effect of stochastic immediate rewards on Composite Q-learning on the chain MDP shown in \Cref{fig:mdp_stochastic}. All actions lead further into the chain, however they yield different expected rewards. Action $a$ has an expected reward of $-0.8$, corresponding to an immediate reward of $-1$ with 80\% chance and 0 with 20\%. Action $b$, on the other hand, has an expected reward of $-1.01$, however it is $+1$ with 99\% chance and $-200$ with 1\%. Composite Q-learning with learning rates of $10^{-2}$ for the full Q-function, $10^{-1}$ for the Shifted Q-functions and $10^{-3}$ for the Truncated Q-functions is compared to Q-learning with the same respective learning rates.\\

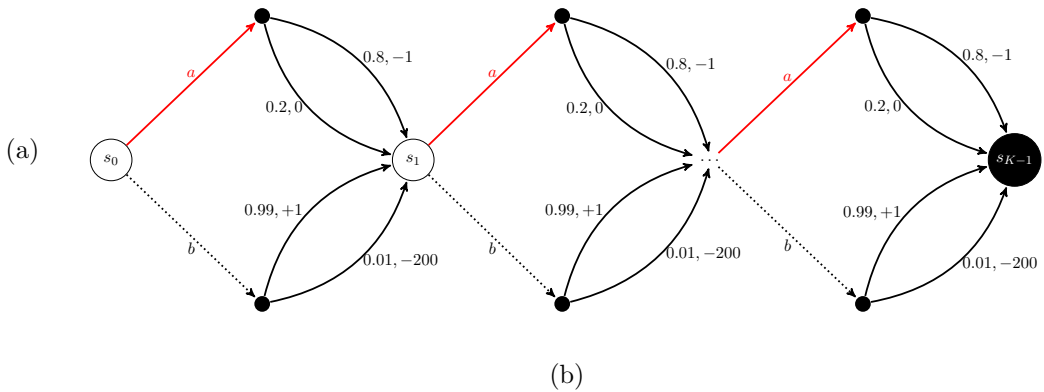
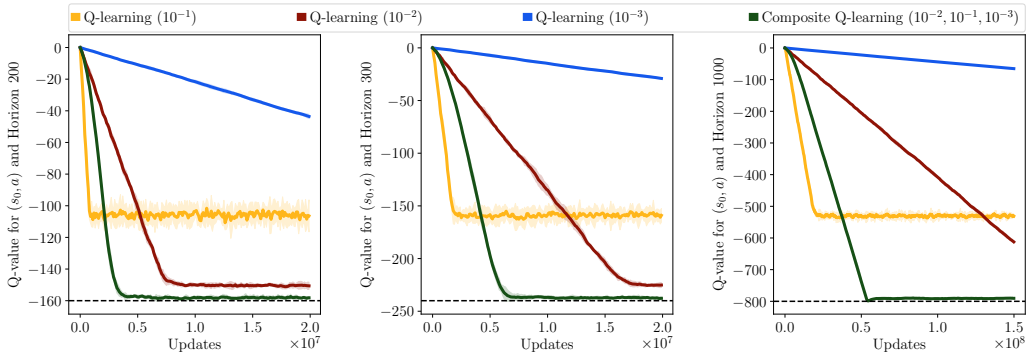
\begin{figure}[h]
\centering
\begin{subfigure}[c]{0.05\textwidth}
\caption{}\label{fig:mdp_stochastic}
\end{subfigure}%
\begin{minipage}[c]{0.95\textwidth}
    \resizebox{0.95\textwidth}{!}{
\begin{tikzpicture}[->,>=stealth',shorten >=1pt,auto,node distance=2.8cm,
                    semithick]
  \tikzstyle{every state}=[]

  \node[state] (0)                    {$s_0$};
  \node[state] (1) [right=6cm of 0]   {$s_1$};
  \node (2) [right=6cm of 1]   {\dots};
  \node[state] (3) [right=6cm of 2, fill=black, text=white]   {$s_{K-1}$};
  
  \node        (mid1) at ($(0)!0.5!(1)$)    {};
  \node[shape=circle,fill=black] (t1) [above=3cm of mid1]  {};
  \node[shape=circle,fill=black] (b1) [below=3cm of mid1]  {};

  \node        (mid2) at ($(1)!0.5!(2)$)    {};
  \node[shape=circle,fill=black] (t2) [above=3cm of mid2]  {};
  \node[shape=circle,fill=black] (b2) [below=3cm of mid2]  {};

    \node        (mid3) at ($(2)!0.5!(3)$)    {};
  \node[shape=circle,fill=black] (t3) [above=3cm of mid3]  {};
  \node[shape=circle,fill=black] (b3) [below=3cm of mid3]  {};
  

  \path (0) edge [red, very thick] node [above] {$a$} (t1)
        (0) edge [very thick, dotted] node [below] {$b$} (b1)
        (b1) edge [bend left, left, very thick] node {$0.99, +1$} (1)
        (b1) edge [bend right, right, very thick] node {$0.01, -200$} (1)
        (t1) edge [bend left, right, very thick] node {$0.8, -1$} (1)
        (t1) edge [bend right, left, very thick] node {$0.2, 0$} (1)

        (1) edge [red, very thick] node [above] {$a$} (t2)
        (1) edge [very thick, dotted] node [below] {$b$} (b2)
        (b2) edge [bend left, left, very thick] node {$0.99, +1$} (2)
        (b2) edge [bend right, right, very thick] node {$0.01, -200$} (2)
        (t2) edge [bend left, right, very thick] node {$0.8, -1$} (2)
        (t2) edge [bend right, left, very thick] node {$0.2, 0$} (2)

        (2) edge [red, very thick] node [above] {$a$} (t3)
        (2) edge [very thick, dotted] node [below] {$b$} (b3)
        (b3) edge [bend left, left, very thick] node {$0.99, +1$} (3)
        (b3) edge [bend right, right, very thick] node {$0.01, -200$} (3)
        (t3) edge [bend left, right, very thick] node {$0.8, -1$} (3)
        (t3) edge [bend right, left, very thick] node {$0.2, 0$} (3);

      \node (l0) [left=0.5cm of 0]                    {};
      \node (r0) [right=0.5cm of 3]                    {};
      \node (a0) [above=0.5cm of t1]                    {};
      \node (b0)                  [below=0.5cm of b1]  {};
\end{tikzpicture}
  }
\end{minipage}
\\[1ex]
\begin{subfigure}[c]{0.05\textwidth}
\caption{}\label{fig:stochchainres}
\end{subfigure}
\begin{minipage}[c]{0.95\textwidth}
\includegraphics[width=0.95\textwidth]{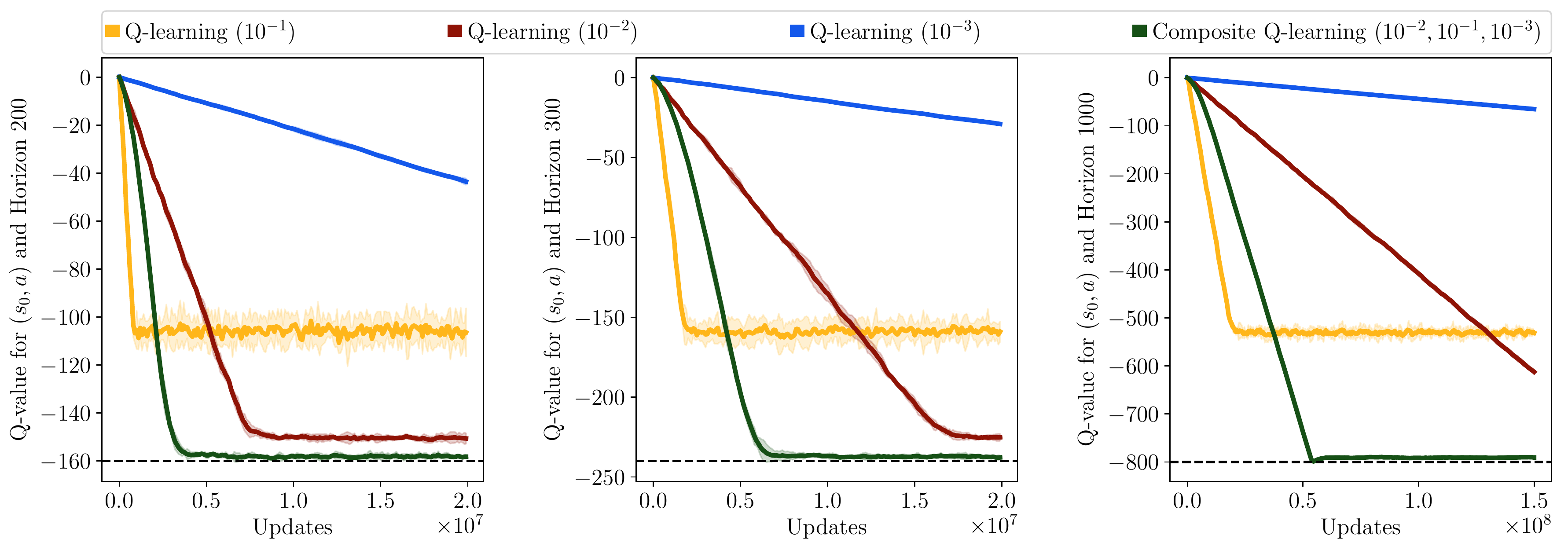}
\end{minipage}
  \caption{(a) In this stochastic chain MDP of horizon $K$, the agent ought to arrive at terminal state $s_{K-1}$ using actions $a$ and $b$. The initial state is $s_0$. Transitions are stochastic. The optimal policy is given in red. (b) Results of Composite Q-learning and Q-learning with different learning rates over 5 training runs for horizons (left) 200, (middle) 300 and (right) 1000 and no discount. The true action-value is indicated by the dashed line.}
\end{figure}

Results for different horizons and no discount are depicted in \Cref{fig:stochchainres}. It can be seen that Q-learning with a high learning rate is converging to a wrong action-value and the deviation becomes more significant the longer the horizon of the task. The same holds for Q-learning with a learning rate of $10^{-2}$, however the deviation becomes smaller. Moreover, Q-learning with the same learning rate as the full Q-function in Composite Q-learning is converging much slower for growing horizons, yet the learning rate for the Truncated Q-functions is 10 times smaller than for Q-learning. This emphasizes the mutual interplay of Truncated and Shifted Q-functions. Furthermore, Composite Q-learning is converging to the true optimal action-value. Q-learning with a small learning rate of $10^{-3}$ is far from achieving a satisfying level of convergence within the time frame of $2\cdot10^7$ updates. \begin{wrapfigure}{r}{0.4\textwidth}
    \centering
    \includegraphics[width=0.4\textwidth]{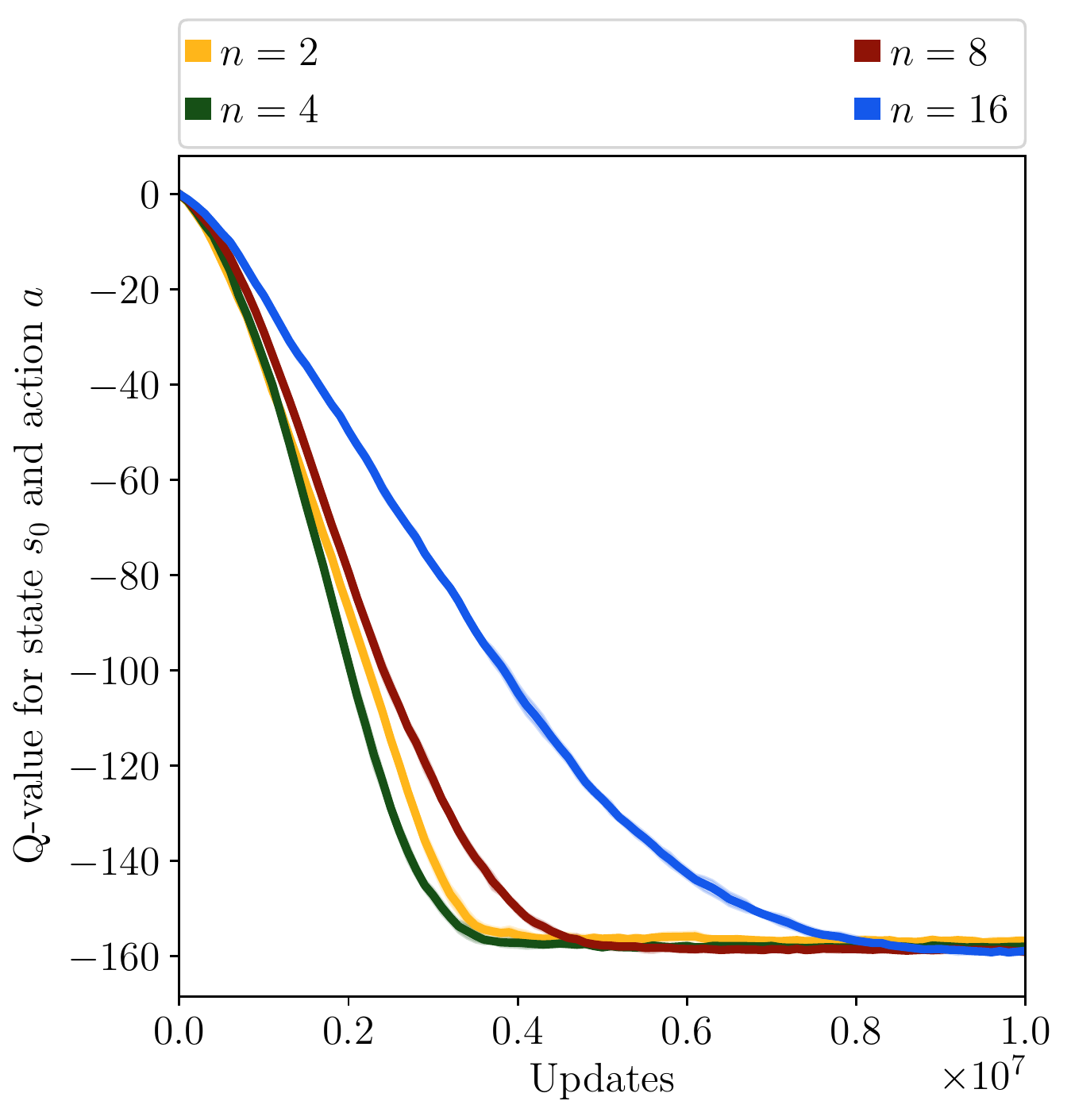}
    \caption{Performance of Composite Q-learning applied to the MDP in \Cref{fig:mdp_stochastic} with different rollout lengths $n$.}
    \label{fig:stochrollout}
\end{wrapfigure}
Shifted Q-functions thus allow for higher learning rates which in combination with a more cautious fitting of short-term predictions leads to increased data efficiency while achieving a smaller deviation from the true action-value. Most importantly, the difference in performance between Q-learning and Composite Q-learning (better value estimation w.r.t. to a high learning rate and better convergence properties comparing to a small learning rate) becomes larger for longer horizons. The performance for different rollout lengths are shown in \Cref{fig:stochrollout}.

\subsection{Composite Q-learning with Function Approximation}
\label{subsec:comfa}
Lastly, we apply Composite Q-learning within TD3 and compare against TD3 and TD3($\Delta$) on three robot simulation tasks of \mbox{OpenAI} Gym \citep{gym} based on MuJoCo \citep{mujoco}: Walker2d-v2, Hopper-v2 and Humanoid-v2. A visualization of the environments is depicted in \Cref{fig:vis}. We first briefly discuss the parameter settings and the hyperparameter optimization in \Cref{subsubsec:parameters}. Then we analyze the properties of Composite Q-learning in terms of data-efficiency and stability on the true reward for Walker2d-v2 and the potential increase of updates per sample for Walker2d-v2 and Hopper-v2 in \Cref{subsubsec:truerew}, before we evaluate Composite TD3 on a very noisy reward signal for Walker2d-v2, Hopper-v2 and Humanoid-v2 in \Cref{subsubsec:noisyrew}.

\begin{figure}[h]
  \centering
  \includegraphics[width=0.3\textwidth]{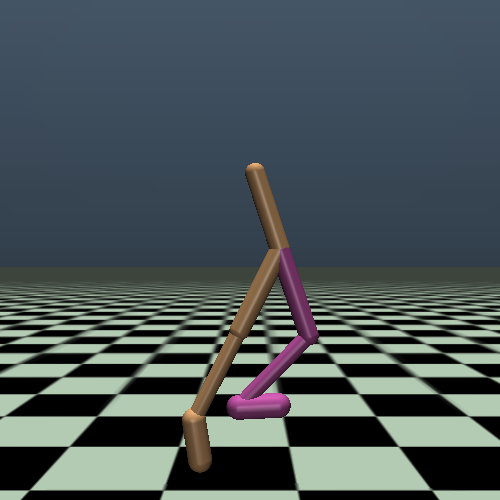}\hspace*{0.3cm}\includegraphics[width=0.3\textwidth]{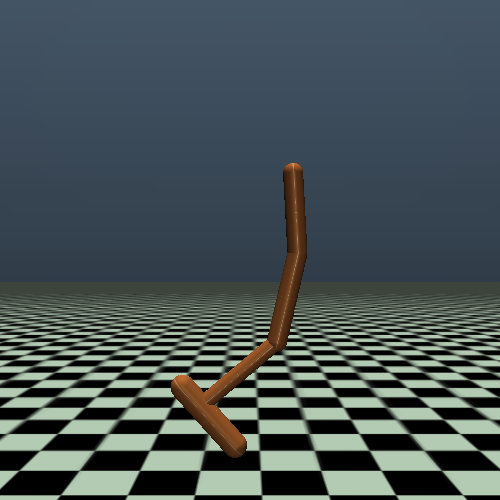}\hspace*{0.3cm}\includegraphics[width=0.3\textwidth]{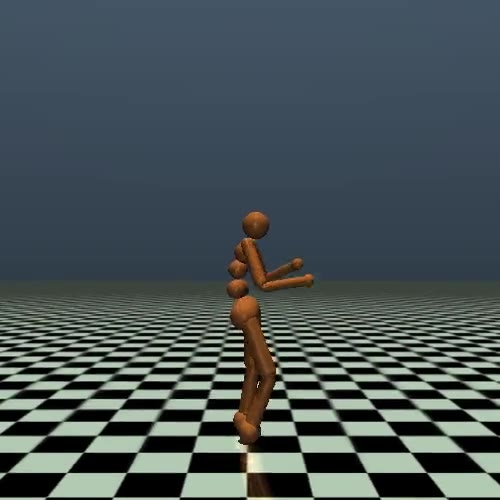}
  \caption{Visualization of Walker2d-v2 (left), Hopper-v2 (middle) and Humanoid-v2 (right).}
  \label{fig:vis}
\end{figure}

\subsubsection{Parameter Setting}
\label{subsubsec:parameters}

For all approaches, we use Gaussian noise with $\sigma=0.15$ for exploration and the optimized learning rate of $10^{-3}$  for the full Q-function. Target update ($5\cdot10^{-3}$) and actor setting (two hidden layers with 400 and 300 neurons and ReLU activation) are set as in \citep{fujimoto2018addressing}. For Humanoid-v2, we use a slightly changed parameter setting with a learning
\begin{wrapfigure}{l}{0.4\textwidth}
    \centering
    \resizebox{0.39\textwidth}{!}{%

\begin{tikzpicture}
\tikzset{>=Latex, line width=1pt}

\tikzstyle{box_dashed}=[rectangle, draw, rounded corners, dashed]
\tikzstyle{box_dotted}=[rectangle, draw, rounded corners, dotted]
\tikzstyle{element}=[draw, fill=black!10, text centered] 
\tikzstyle{element_rect}=[element, rounded corners, rectangle, minimum height = 1cm, minimum width = 1.5cm]
\tikzstyle{element_ell}=[element, ellipse]
\tikzstyle{element_circ}=[element, circle]
\tikzstyle{element_dia}=[element, diamond]
\tikzstyle{element_rect_split}=[element, rounded corners, rectangle split, rectangle split horizontal, rectangle split parts=2]

\node (state) [] {$s_t$};
\node (action) [right=0.5cm of state] {$a_t$};
\node (mid) at ($(state)!0.5!(action)$) {};
\node (layer2) [element_rect, minimum width = 5cm, below=0.7cm of mid] {2 fully connected layers (500, leaky ReLU)};
\node (mid) [below=0.8cm of layer2, minimum height = 1.2cm] {};
\node (layer3a) [element_rect, minimum width = 3cm, left=0.05cm of mid, align=center] {fully connected\\(500, leaky ReLU)};
\node (layer3b) [element_rect, minimum width = 3cm, right=0.05cm of mid] {$Q^{\text{Tr}}_1 Q^{\text{Tr}}_2\dots Q^{\text{Tr}}_n$};
\node (layer4a) [element_rect, minimum width = 2cm, below=0.7cm of layer3a, align=center] {fully connected\\(500, leaky ReLU)};
\node (layer4b) [element_rect, minimum width = 2.5cm, below=0.7cm of layer3b] {$Q^{\text{Sh}}_1 Q^{\text{Sh}}_2\dots Q^{\text{Sh}}_n$};

\node (Q) [element_rect,below=0.7cm of layer4a] {$Q$};

\draw [] (state.south) edge[->, thick, out=270, in=90] (layer2.north);
\draw [] (action.south) edge[->, thick, out=270, in=90] (layer2.north);
\draw [] (layer2.south) edge[->, thick, out=270, in=90] (layer3a.north);
\draw [] (layer2.south) edge[->, thick, out=270, in=90] (layer3b.north);
\draw [] (layer3a.south) edge[->, thick, out=270, in=90] (layer4a.north);
\draw [] (layer3a.south) edge[->, thick, out=270, in=90] (layer4b.north);
\draw [] (layer4a.south) edge[->, thick, out=270, in=90] (Q.north);
\end{tikzpicture}
    }%
    \caption{Optimized architecture of the Composite-Q network used in our experiments.}
    \label{fig:architecture}
\end{wrapfigure}
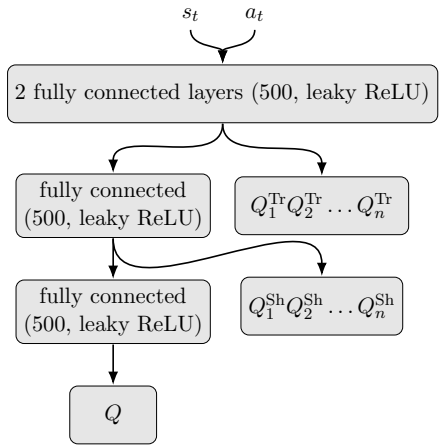
rate of $10^{-4}$ for both actor and critic as suggested in \citep{dorka2020dynamic}. For Composite Q-learning, we calculate the full Q-values, as well as Truncated and Shifted-Q values in one, combined, architecture for improved efficacy. Additionally, we found that it is beneficial to estimate the Truncated and Shifted Q-values in different layers, as depicted in \Cref{fig:architecture}. For all parameters in the layers prior to the full Q-output, we use the parameter setting as described above.\\

Hyperparameters are optimized for all approaches, including the baselines TD3 and TD3($\Delta$), for the given subset of MuJoCo tasks via Random Search over $10$ training runs with the configuration space shown in \Cref{tab:hyperopt}. Learning rates for the corresponding output layers of the Truncated and Shifted Q-values are optimized individually. We use a learning rate of $6\cdot10^{-5}$ for the Truncated Q-functions and $5\cdot10^{-3}$ for the Shifted Q-functions. For the noisy experiments, as well as for the experiments with multiple gradient steps per collected sample, we set $n=4$. For the Walker2d-v2 experiment on the true reward function, we set $n=10$. In terms of regularization, we set $\beta_\text{Tr}=0.002$ and $\beta_\text{Sh}=0.001$. For Humanoid-v2, we set the optimized learning rates one magnitude lower. When executing multiple gradient steps per collected sample, we keep the learning rate of the Shifted Q-functions at $10^{-3}$. We further evaluate TD3 with the same maximum learning rate. We optimize the number of layers and the number of neurons per layer for the critic in Composite Q-learning (4 layers with 500 neurons and leaky ReLU activation) and the architecture for the critic in all other approaches (two layers with 500 neurons and leaky ReLU activation). For an overview of the architecture of the Composite-Q network see \Cref{fig:architecture}. For TD3($\Delta$), we use the $\gamma$-schedule as suggested by \citeauthor{timescales}, i.e. $\gamma_1=0$ and $\gamma_{i>1}=\frac{\gamma_{i-1}+1}{2}$, with an upper limit of 0.99.

\begin{table}[h]
    \centering
    \caption{Configuration space of the hyperparameter optimization. For the full Q-function in all approaches, we used the optimized learning rate of $10^{-3}$ as in \citep{fujimoto2018addressing}. Hyperparameters denoted by $^*$ were optimized individually for all approaches.}
    \begin{tabular}{cc}
        \toprule
        Hyperparameter &  Values\\
        \midrule
        number of layers$^*$ & $\{2,3,4\}$\\
        number of neurons$^*$ & $\{300, 400, 500\}$\\
        activation$^*$ & $\{\text{ReLU}, \text{leaky ReLU}\}$\\
        $\alpha_\text{Tr}$&$\{6\cdot10^{-6}, 10^{-5}, 6\cdot10^{-5}, 10^{-4}, 10^{-3}\}$\\
        $\alpha_\text{Sh}$&$\{10^{-3}, 2\cdot10^{-3}, 5\cdot10^{-3}, 10^{-2}\}$\\
        $\beta_\text{Tr}$&$\{10^{-3}, 2\cdot10^{-3}, 4\cdot10^{-3}\}$\\
        $\beta_\text{Sh}$&$\{5\cdot10^{-4}, 10^{-3}, 2\cdot10^{-3}\}$\\
        $n$&$\{3,4,10,15,20\}$\\
        \bottomrule
    \end{tabular}
    \label{tab:hyperopt}
\end{table}

\subsubsection{True Reward Function}
\label{subsubsec:truerew}

We compare Composite TD3, TD3($\Delta$) and conventional TD3 with both default and high learning rate (TD3 HL) on the true reward for the Walker2d-v2 environment in \Cref{fig:singlewalker}. Composite TD3 and TD3($\Delta$) yield a better and more stable performance than TD3 with either learning rate, with Composite TD3 being the best performing of all approaches. The differences between conventional TD3 and Composite TD3 are significant, as shown in \Cref{tab:significancenonoise}. TD3 is not able to gain performance with higher learning rate and even degenerates. Both compositional Q-learning approaches show less variance.\\

\begin{figure}[h]
  \centering
  \begin{subfigure}{0.33\textwidth}
    \includegraphics[width=\textwidth]{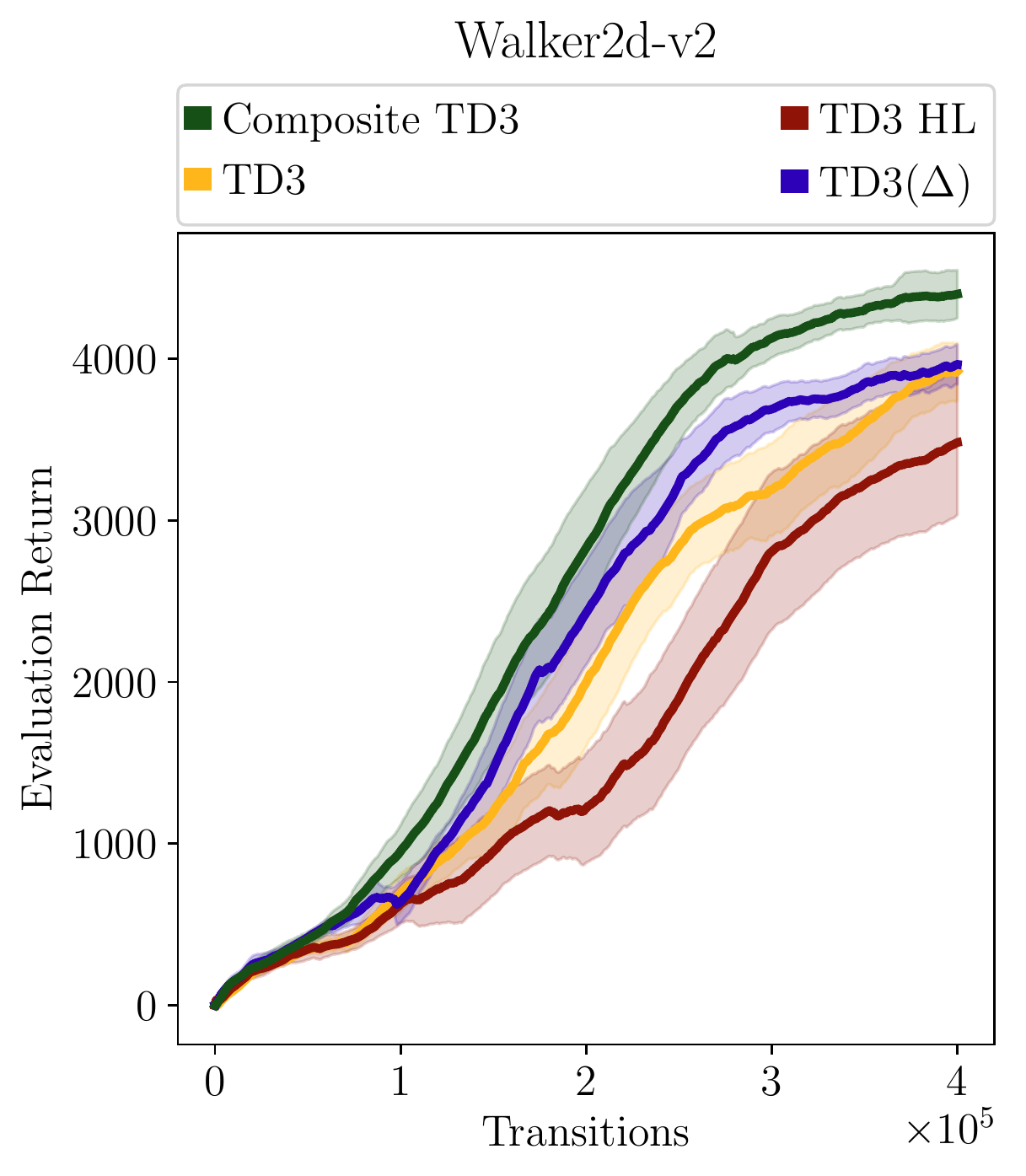}
    \caption{}
    \label{fig:singlewalker}
  \end{subfigure}
  \begin{subfigure}{0.66\textwidth}
    \includegraphics[width=\textwidth]{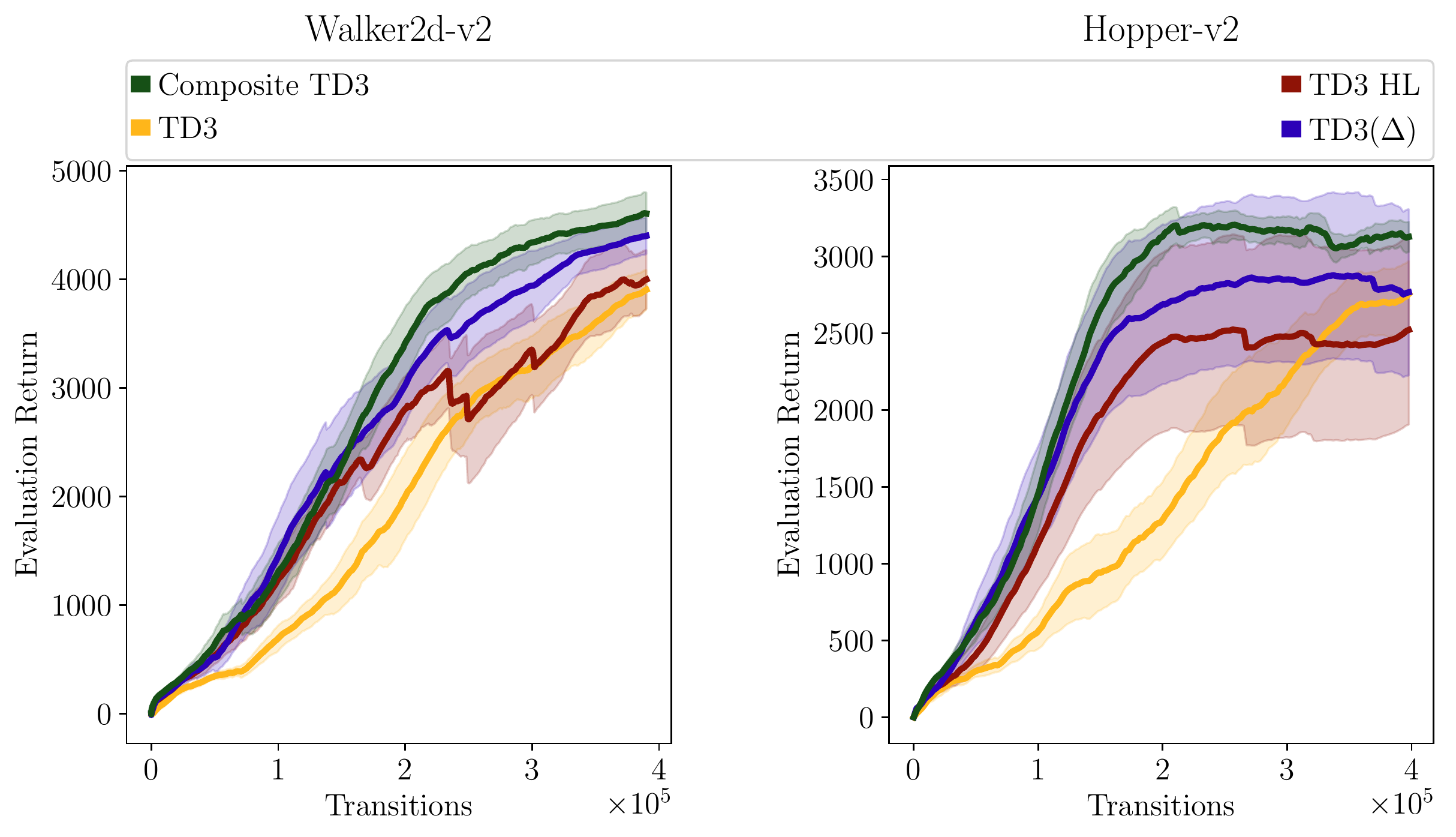}
    \caption{}
    \label{fig:mgs}
  \end{subfigure}
  \caption{(a) Mean performance and half a SD for the Walker2d-v2 environment with the default reward. (b) Results of Composite Q-learning on the vanilla reward function and multiple update steps per collected sample for the Walker2d-v2 and Hopper-v2 environment.}
  \label{fig:seven}
\end{figure}

\begin{table}[h]
    \centering
    \caption{p-values for Welch's t-test on the mean area under the learning curve for 8 different runs of Composite TD3 in the noisy reward experiments. Subparts (a) and (b) correspond to those in \Cref{fig:seven}.}
    \begin{tabular}{cccc}
        \toprule
         &TD3&TD3 (HL)& TD3($\Delta$)\\
         \midrule
         Walker2d-v2 (a)&$8\cdot10^{-3}$&$7\cdot10^{-4}$&$6\cdot10^{-2}$\\
         Walker2d-v2 (b)&$3\cdot10^{-4}$&$2\cdot10^{-2}$&$5\cdot10^{-1}$\\
         Hopper-v2 (b)&$3\cdot10^{-5}$&$4\cdot10^{-2}$&$5\cdot10^{-1}$\\
         \bottomrule
    \end{tabular}
    \label{tab:significancenonoise}
\end{table}

The influence of the horizon of Truncated and Shifted Q-functions, as well as the respective regularization weights maximizing and minimizing the entropy, is given exemplary for the Walker2d-v2 environment in \Cref{fig:betaroll}. It can be seen that the regularization of the Shifted Q-functions has a higher influence on the performance. If the weight is set too high, the variance increases significantly. The same holds for a too long temporal horizon. The regularization of the Truncated Q-functions has less influence on convergence and variance, however, please note the strong connection of these hyperparameters. Therefore, the best triad setting needs to be found for a given problem and yields an interesting direction for meta-reinforcement learning. Prior experiments with other regularization terms only minimizing the entropy also showed enhanced stability, however, it is important to prevent the Shifted Q-functions from overfitting when applying higher learning rates in the TD3 setting.\\

\begin{figure}[t]
  \centering
  \includegraphics[width=\textwidth]{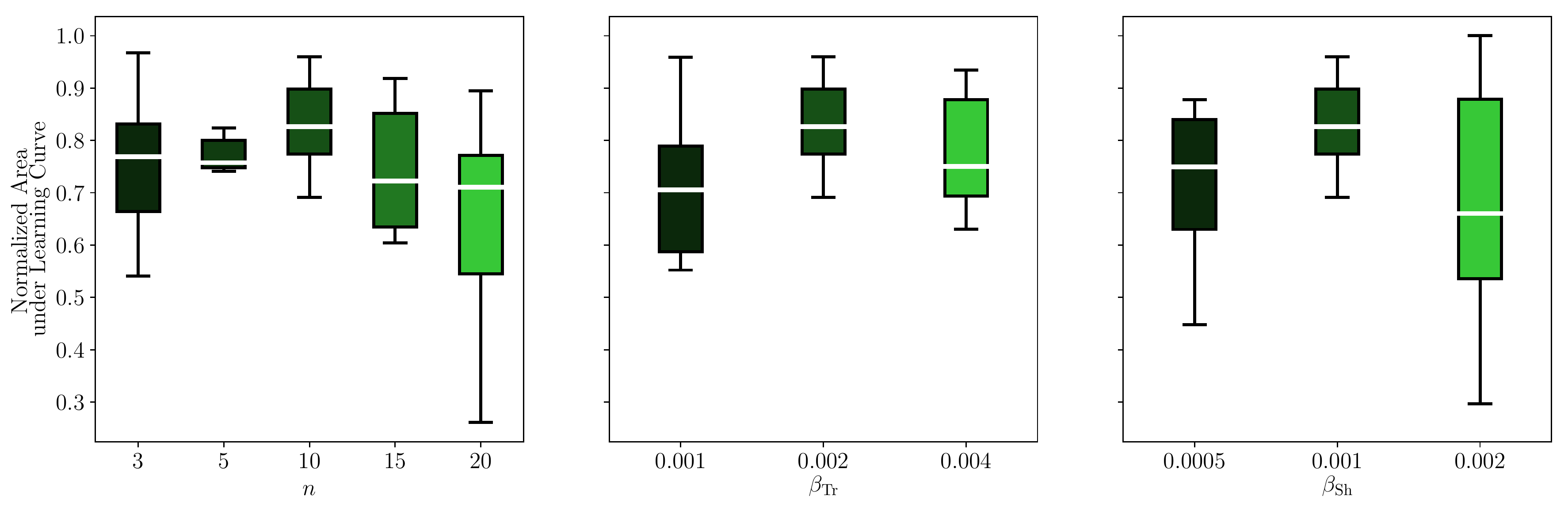}
  \caption{Normalized area under the learning curve for Composite TD3 in the Walker2d-v2 environment with different truncation horizons $n$ (left) and different regularization weights $\beta^\text{Tr}$ (middle) and $\beta^\text{Sh}$ (right). The plots show median and interquartile ranges over 8 training runs, each representing mean evaluation performance over 100 initial states.}
  \label{fig:betaroll}
\end{figure}

In the last experiments under the true reward, we investigate the potential increase in data-efficiency when applying multiple gradient steps per collected sample. Results can be seen in \Cref{fig:mgs}, with corresponding significance tests provided in \Cref{tab:significancenonoise}. While there seems to be a limit for TD3 to gain new knowledge from multiple gradient steps, both compositional Q-learning approaches yield a faster learning speed.

In addition, Composite Q-learning has no increase in variance opposed to all other approaches. The high learning rate in combination with multiple gradient steps is harmful w.r.t. learning stability.

\subsubsection{Noisy Reward Function}
\label{subsubsec:noisyrew}
\begin{figure}[t]
  \centering
  \includegraphics[width=\textwidth]{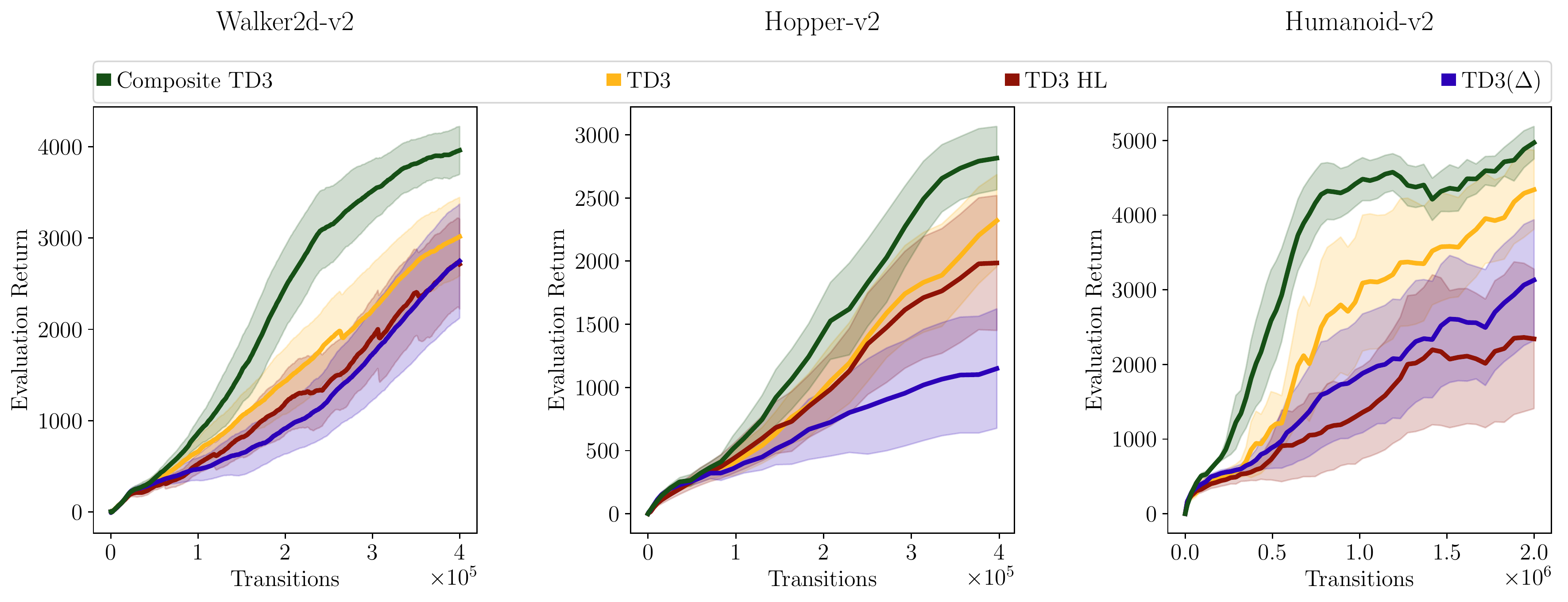}
  \caption{Mean performance and half a SD over 8 runs for (left) Walker2d-v2, (middle) Hopper-v2 and (right) Humanoid-v2 with uniform noise on the reward function as in \citep{timescales}.}
  \label{fig:noisyreward}
\end{figure}
\begin{table}[t]
    \centering
    \caption{p-values for Welch's t-test on the mean area under the learning curve for 8 different runs of Composite TD3 in the noisy reward experiments.}
    \begin{tabular}{cccc}
        \toprule
         &TD3&TD3 (HL)& TD3($\Delta$)\\
         \midrule
         Walker2d-v2&$4\cdot10^{-5}$&$2\cdot10^{-8}$&$6\cdot10^{-8}$\\
         Hopper-v2&$8\cdot10^{-3}$&$8\cdot10^{-3}$&$2\cdot10^{-6}$\\
         Humanoid-v2&$3\cdot10^{-2}$&$4\cdot10^{-4}$&$5\cdot10^{-6}$\\
         \bottomrule
    \end{tabular}
    \label{tab:significance}
\end{table}
We evaluate Composite TD3 under a noisy reward function and compare to TD3, TD3 with high learning rate and TD3($\Delta$). The immediate reward is replaced by a uniform sample $u\sim\mathcal{U}[-1, 1]$ with 40\% chance. This corresponds to the strongest noise level described in \citep{timescales}. In order to account for the high variance, we provide mean and standard deviation over 8 runs. Results can be seen in \Cref{fig:noisyreward}, corresponding significance tests are provided in \Cref{tab:significance}. Composite TD3 seems to be very robust, even for the very complex Humanoid-v2 environment. All other approaches suffer from slower learning speed and high variance. This holds especially for TD3 with high learning rate and TD3($\Delta$). We believe this to be caused by overfitting to the noisy reward function due to the high learning rate in TD3 and due to the low gamma of the first value-functions in TD3($\Delta$). Since all other predictions only add to the prediction of the preceding head, there is nothing to prevent the overestimation bias from being propagated to the other outputs of higher discount value. In Composite Q-learning however, Shifted and Truncated Q-functions estimate distinct parts of the temporal chain. This is also highlighted in \Cref{fig:errors}, where the TD-errors of Shifted and Truncated Q-functions can be seen along with learning curves of Humanoid-v2. Whilst in the default setting with vanilla reward the Truncated Q-functions are converging much faster than the Shifted Q-function due to the simpler problem induced by the smaller horizon, the opposite holds in case of the noisy reward. In these experiments, the Truncated Q-functions are not capable of an accurate prediction within the given time frame, yet the Shifted Q-functions can construct an accurate chain given those inaccurate predictions reliably leading to a well-performing policy. This backs our hypothesis that Shifted Q-functions only need to smooth out the variance of the transitions whereas Truncated Q-functions have to account for variance in the reward as well, which adds to the complexity of the problem significantly. By breaking down the time scales of the long-term return, Composite Q-learning makes a huge step towards the real-world application of Q-learning where a reliable value of the immediate reward can often not be provided. This translates to faster learning, which is also shown in \Cref{tab:results}.\\

\begin{figure}[t]
  \centering
  \begin{subfigure}{0.49\textwidth}
    \includegraphics[width=\textwidth]{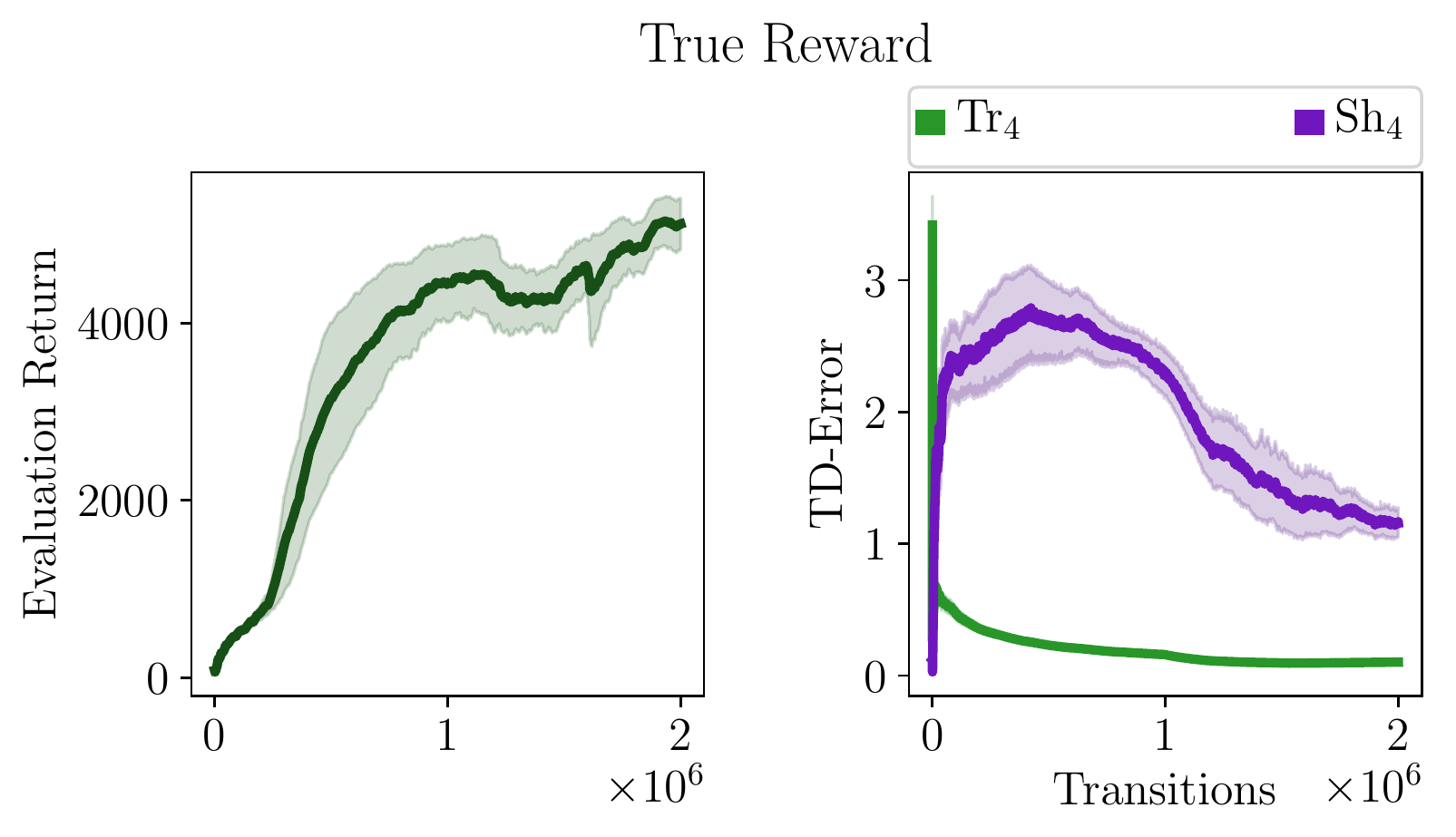}
    \caption{}
  \end{subfigure}
  \begin{subfigure}{0.49\textwidth}
    \includegraphics[width=\textwidth]{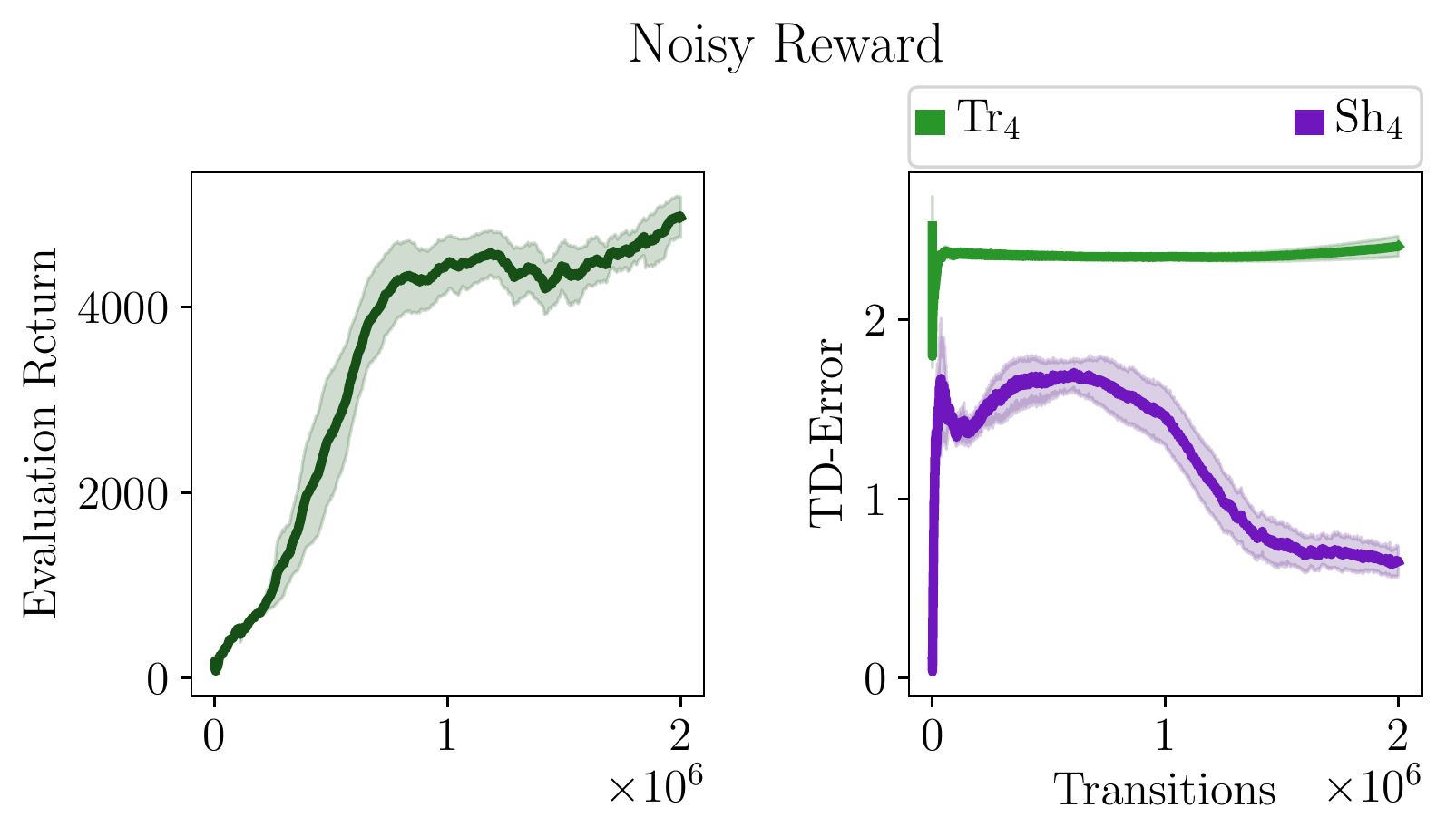}
    \caption{}
  \end{subfigure}
  \caption{Performance and TD-errors on Humanoid-v2 of Truncated and Shifted Q-functions at $n=4$ w.r.t. the (a) vanilla and (b) noisy reward. Please note that \textit{TD-error} here means the deviation from the associated target.}
  \label{fig:errors}
\end{figure}

\begin{table}[h]
  \caption{Mean normalized area under the learning curve and SD over 8 training runs of Composite TD3 in the noisy reward experiments.}
  \small
  \centering
  \begin{tabular}{ccccc}
    \toprule
    Method&Walker2d-v2&Hopper-v2&Humanoid-v2\\
    \midrule
    TD3&$68\%\pm22\%$&$76\%\pm28\%$&$68\%\pm30\%$\\
    TD3 (HL) &$56\%\pm18\%$&$72\%\pm37\%$&$38\%\pm29\%$\\
    TD3($\Delta$)&$52\%\pm23\%$&$49\%\pm31\%$&$56\%\pm26\%$\\
    \midrule
    Composite TD3&$\bm{100\%\pm20\%}$&$\bm{100\%\pm25\%}$&$\bm{100\%\pm16\%}$\\
    \bottomrule
  \end{tabular}
  \label{tab:results}
\end{table}
\Cref{tab:maxreturn} shows the mean maximum return achieved by Composite TD3, TD3 with default and high learning rate and TD3($\Delta$). Composite TD3 reaches significantly higher returns with lowest variance compared to vanilla TD3 when applied to a noisy reward function, especially for the most complex environment Humanoid-v2.
\begin{table}[h]
  \caption{Mean maximum return and SD over 8 training runs for the noisy reward experiments.}
  \small
  \centering
  \begin{tabular}{cccc}
    \toprule
    Method&Walker2d-v2&Hopper-v2&Humanoid-v2\\
    \midrule
    TD3&$3063\pm813$&$2386\pm729$&$4453\pm1070$\\
    TD3 (HL) &$2950\pm866$&$2124\pm1054$&$2506\pm2004$\\
    TD3($\Delta$)&$2772\pm1245$&$1242\pm954$&$3251\pm1687$\\
    \midrule
    Composite TD3&$\bm{4041\pm476}$&$\bm{2931\pm457}$&$\bm{5019\pm404}$\\
    \bottomrule
  \end{tabular}
  \label{tab:maxreturn}
\end{table}

Our experimental results show that Composite Q-learning offers a significant improvement over traditional Q-learning methods and a great potential for real-world applications where the extraction of the immediate reward signal can be noisy due to complex training setups or fitted reward functions.

\section{Conclusion}
\label{sec:conclusion}
We introduced Composite Q-learning, an off-policy reinforcement learning method that divides the long-term value into smaller time scales. It combines Truncated Q-functions acting on a short horizon with Shifted Q-functions for the remainder of the rollout. We analyzed the efficacy of Composite Q-learning in the tabular case and showed that the benefit of long-term composition of short-term predictions increases with growing task horizon. Additionally, we evaluated the performance of Composite Q-learning in the deep case. As a baseline, we further introduced and evaluated TD3($\Delta$), an off-policy variant of TD($\Delta$). We showed on three simulated robot tasks that compositional Q-learning methods can be of advantage w.r.t. data-efficiency compared to vanilla Q-learning methods and that Composite TD3 outperforms vanilla TD3 by 24\% - 32\% in terms of area under the learning curve. In addition, Composite Q-learning proved to be very robust to noisy reward signals. Going forward, the uncertainty estimate based on the variance of all predictions from the Composite Q-function could be of benefit in update calculation, transition sampling and exploration. We further believe our method to be a better fit for non-stationary reward functions than traditional Q-learning methods due to the flexibility provided by the divided long-term return. In addition, the formulation of Shifted Q-functions could serve as an implicit dynamics model to further build the bridge between model-based and model-free methods.

\bibliography{composite_q}

\end{document}